%% file: main.tex
\documentclass{article}

\usepackage[preprint]{corl_2026} %
\input{headers/preamble}
\input{headers/gls}
\input{headers/notation}

\title{AI Coaching for Accelerating Human Skill Development with Reinforcement Learning}

\author{
  Wei Wang$^1$~~~~Enlin Gu$^1$~~~~Antonio Loquercio$^1$~~~~Haimin Hu$^{\dagger 1,2}$~~~~Rahul Mangharam$^{\dagger 1}$\\
  $^1$University of Pennsylvania~~~~$^2$Johns Hopkins University\\
  \texttt{\{wang100,guenlin,aloque,rahulm\}@engineering.upenn.edu},~~~~\texttt{haimin@cs.jhu.edu} \\ \vspace{-0.2in}
}

\begin{document}

\makeatletter
\let\@oldmaketitle\@maketitle%
\renewcommand{\@maketitle}{\@oldmaketitle%
\vspace{-2em}
\centering
\includegraphics[width=\textwidth]{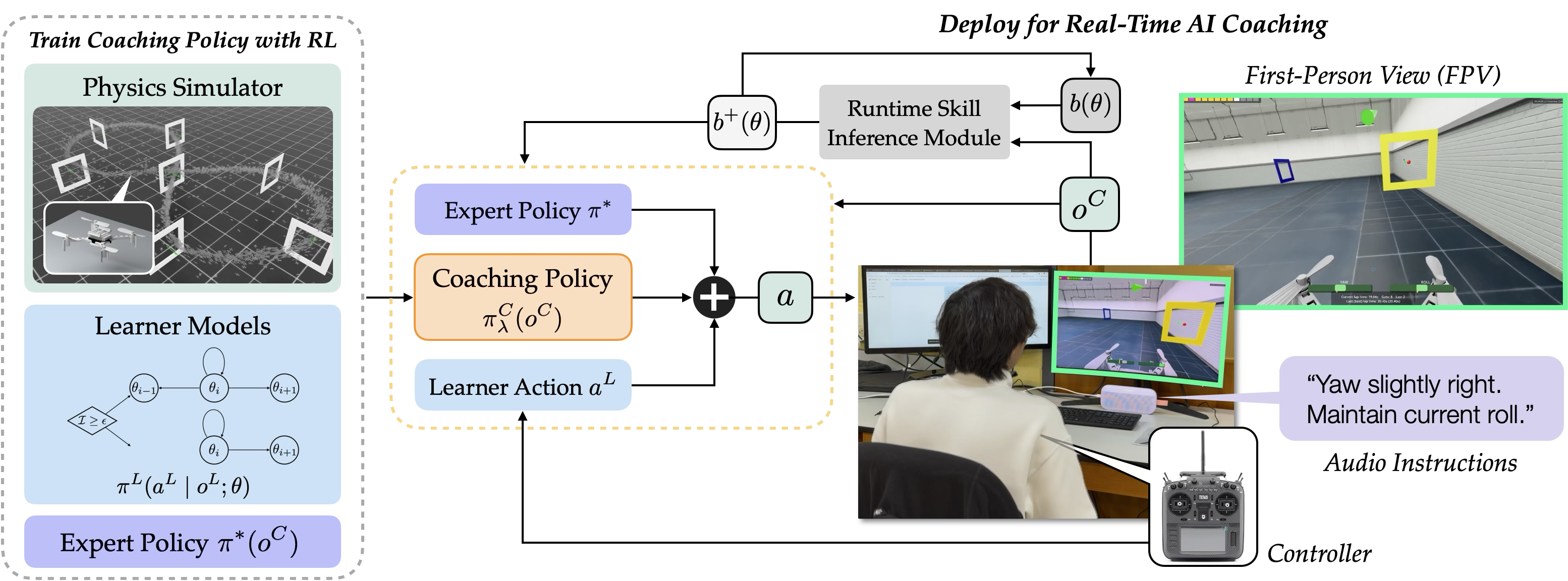}
\captionof{figure}{
Our AI coach accelerates human motor-skill development through strategic scaffolding and stepping back.
\textit{Left:} We train the AI coach using RL by augmenting a physics simulator with a learner action model, a probabilistic automaton governing skill evolution, and a pretrained expert policy.
\textit{Center:} The AI coach adaptively blends the expert action with the learner's command to produce an executed action $\action$.
\textit{Right:} In live coaching of FPV drone racing, the AI maintains a belief $\bel(\hstate)$ over the learner's latent skill and conveys feedback through multimodal channels, including shared control, audio instructions, and visual cues.}
\label{fig:front}
\vspace{-1em}
\bigskip}
\makeatother
\maketitle
\blfootnote{\textsuperscript{$\dagger$}\,Equal advising.}
\blfootnote{Project webpage: \href{https://ai-coaching-drone-racing.github.io/}{\textcolor{darkblue}{\nolinkurl{https://ai-coaching-drone-racing.github.io/}}}}

\begin{abstract}
AI copilots can substantially boost human performance through shared control, but excessive assistance can induce over-reliance and skill atrophy.
This paper studies how an embodied AI agent can act as a \emph{coach} that accelerates human motor-skill development.
We argue that effective coaching requires strategic scaffolding and stepping back that are aligned with the learner's capability, allowing productive failures that drive learning.
We formalize the interactive AI coaching process as a non-cooperative dynamic game in which the learner optimizes task performance while the coach targets the learner's independent competence.
Building on this formalism, we develop a reinforcement learning framework combining adaptive shared control with probabilistic models of the coach's causal influence on skill evolution, enabling tractable training of coaching policies.
A comprehensive user study ($N=33$) on first-person-view drone racing shows significant gains in human learning outcomes over state-of-the-art AI coaching baselines.
\end{abstract}


\section{Introduction}

Mastering a motor skill, from racing a car at the limit, to piloting a drone through tight gates, requires deliberate practice at the boundary of one's current ability.
Human coaches accelerate this process not by performing the task for the learner, but by orchestrating the right challenges at the right moments~\cite{kapur2008productive,metcalfe2017learning}.
Can an embodied AI agent assume this role?

\begin{figure}[!tbp]
    \centering
    \includegraphics[width=\linewidth]{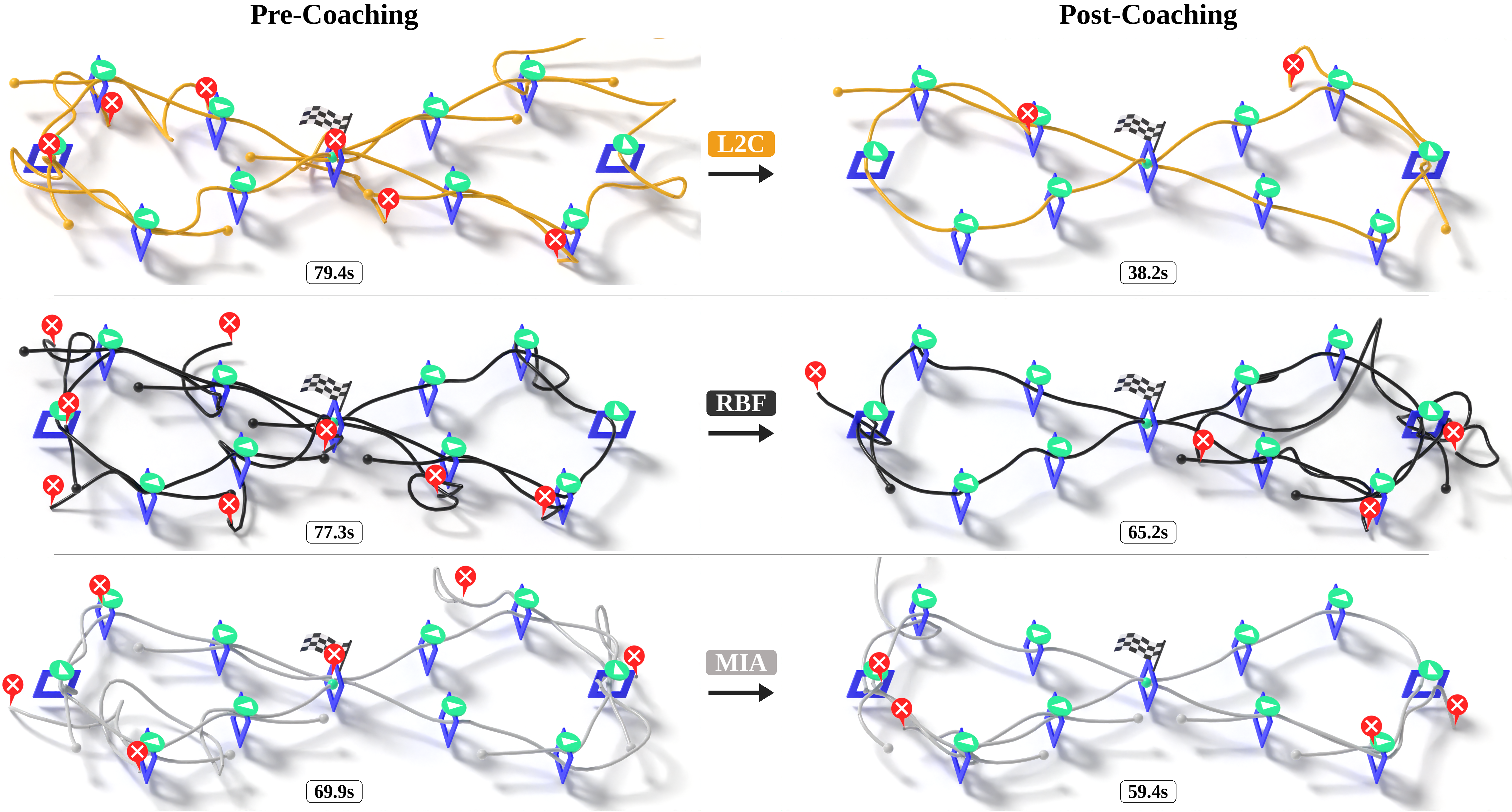}
    \caption{Learning to Coach (L2C) accelerates human skill development. Example pre- and post-coaching trajectories under three coaching methods. The flag marks the start and finish line, the crosses mark failures, and arrows above each gate indicate the flight direction. After just 40 minutes of training, our L2C coach reduced lap time by 27.9\% on average, substantially outperforming both baselines (RBF: 11.3 \%, MIA: 6.2 \%). 
    }
    \vspace{-1.5em}
    \label{fig:human_traj}
\end{figure}

Recent advances in robot learning have produced embodied AI agents that can outperform world champions~\cite{wurman2022outracing,kaufmann2023champion}.
A natural use of such expert AI is as a \emph{shared-control copilot}: a system that fuses the human's commands with the AI's corrective inputs to keep the team performing well despite human error~\cite{reddy2018shared,decastrodreaming,srivastava2025shared,oh2025safety,sha2026efficient}.
This is the dominant paradigm in modern human--AI collaboration, where AI acts as a ``guardian angel'' that prevents failures and smooths execution.
However, helping humans succeed now is fundamentally different from helping them succeed alone in the future.
Persistent AI assistance can induce over-reliance and skill stagnation in human users~\cite{bastani2024generative,macnamara2024does}, depriving them of the productive failures~\cite{kapur2008productive} that drive skill development and retention.
This is also a phenomenon that recent AI safety work characterizes more broadly as \textit{gradual disempowerment}~\cite{kulveit2025gradual}.
Conversely, a coach that withdraws assistance indiscriminately exposes the learner to failures that are uninformative or frustrating, neither of which benefits human skill learning.

\p{Contributions}
\textit{Our key insight is that an effective AI coach must neither overprotect nor disengage from the learner; it should strategically step back when doing so turns mistakes into catalysts of long-term independent competence.}
We formalize this through the \emph{Human--AI Coaching Game}, a novel class of non-cooperative dynamic games in which a learner optimizes task performance while a coach targets the learner's long-term independent competence.
We develop a reinforcement learning framework combining an adaptive, closed-loop shared control rule, a probabilistic automaton that captures the coach's causal influence on learner skill, and a tractable skill-improvement reward for scalable coaching policy training.
We deploy our AI coach in a user study ($N=33$) of \gls{FPV} drone racing with a high-fidelity simulator, demonstrating significant gains in human learning (reduced lap times and failure counts) over state-of-the-art AI coaches.
To our knowledge, this is the first theory and algorithm that explicitly treat human motor-skill development as the main objective of AI agents, thereby transforming them from passive guardians into personalized mentors.

\p{Related Work}
As embodied AI agents become more capable, they have been increasingly deployed to assist humans in challenging tasks, from drone piloting~\cite{reddy2018shared,backman2023reinforcement} to high-speed car racing~\cite{decastrodreaming,oh2025safety}, through \textit{shared control}: AI augments human commands with its corrective inputs.
While effective at enhancing task performance, recent work has shown that persistent AI assistance may induce over-reliance, leaving the human's skill stagnant or degraded~\cite{bastani2024generative,macnamara2024does,kulveit2025gradual}.

A growing line of research thus asks whether AI can act as a \textit{mentor} that empowers humans to develop and retain skills they can exercise independently.
Closest to our work is Shen et al.~\cite{shen2025cyber}, which gradually decays the assistance level in shared control to reduce learner over-reliance. While conceptually similar to our productive-failure idea, it fades assistance on a \emph{fixed} rule agnostic to both learner skill and physical context. 
Z-COACH~\cite{srivastava2025shared} uses shared autonomy to identify which sub-skills lie within the learner's Zone of Proximal Development~\cite{vygotsky1978development}, but assumes a fixed assistance level and focuses on curriculum design rather than synthesizing a coaching strategy that accelerates human learning on a given task.
Our work differs from these by closing the loop on both the learner's evolving skill and the runtime physical context, learning a coaching policy that adaptively decides when to scaffold and when to step back, and naturally complements curriculum-design approaches.

Belief-space dynamic games~\cite{sadigh2018planning,schwarting2021stochastic,hu2023deception} provide a general mathematical framework for human--robot interaction, in which a robot acts under uncertainty about latent human states such as intent, preferences, or skill.
A widely studied \textit{cooperative} variant of belief game is Assistance Games~\cite{fern2014decision,hadfield2016CIRL,fisac2020PPVA,laidlaw2025assistancezero}, in which a robot learns and acts on the human's latent objective to better assist the human.
To date, however, this framework has been used almost exclusively for \emph{humans teaching robots}, not the other way around.
We extend belief-space dynamic games to AI coaching by introducing a novel \textit{non-cooperative} game that captures the AI's fundamental role as a mentor for long-term human skill development rather than solely maximizing the task performance.

\section{Formalism: The Human--AI Coaching Game}
\label{sec:formulation}

\p{Motivation: The Dilemma in Coaching}
Consider an AI coach helping a learner practice a motor skill, \eg, flying a racing drone through gates (\autoref{fig:front}).
At each moment, the coach must decide how much it \emph{scaffolds} the learner toward task completion, or \emph{steps back} to let them struggle, fail, and learn.
This decision raises a dilemma: intervening aggressively risks \textit{coddling}, leaving the learner over-reliant on AI; withdrawing too early yields uninformative failures and frustration.
Moreover, the learner's true skill is never directly observed and must be inferred from behavior as it evolves.

\p{Human--AI Coaching Game}
Formally, we cast the AI coaching problem as a two-player \gls{POSG}~\cite{hansen2004dynamic}, which we refer to as a Human--AI \ourgame{}.
We argue that coaching is fundamentally a \textit{non-cooperative} game~\cite{basar1998dynamic} where the human and AI have different objectives: 
the learner seeks to maximize the task performance,
while the coach optimizes a pedagogical objective, namely the learner's \emph{independent} competence.
These distinct incentives are critical for coaching success. 
The pedagogical reward encourages the coach to proactively modulate assistance (\eg, scaffolding or stepping back) when doing so accelerates long-term human skill development, even at the cost of short-term task performance.

\begin{definition}
\label{def:coaching_game}
A \ourgame{} is a two-player \gls{POSG} between a human learner~$\human$
and an AI coach~$\ai$, defined by the tuple
$\game = \bigl(\sset,\aset,\transprob,
\hsset,\obsset^\human,\obsset^\ai,\obsfunc^\human,\obsfunc^\ai,
\reward^\human,\reward^\ai,\prob_0,\discount\bigr)$, where:
\begin{itemize}[label={}, topsep=0pt, parsep=0pt, partopsep=0pt]
    \item $\sset$ is the set of physical states: $\state \in \sset$;
    \item $\aset$ is the action set shared by the learner and the coach: $\action^\human, \action^\ai \in \aset$;
    \item $\hsset \subset [0,1]$ is a finite, totally ordered set of latent skill levels: $\hstate \in \hsset$, with $\hstate = 0$ denoting lowest proficiency and $\hstate = 1$ expert-level mastery;
    \item $\transprob(\state', \hstate' \mid \state, \hstate, \action^\human, \action^\ai)$
          is the joint state--skill transition distribution;
    \item $\obsset^\human, \obsset^\ai$ are the observation sets
          for the learner and coach: $\obsh \in \obsset^\human, \obsr \in \obsset^\ai$;
    \item $\obsfunc^\human(\state),
      \obsfunc^\ai(\state, \action^\human)$
      are the observation functions;
    \item $\reward^\human(\state, \action^\human, \action^\ai)$
          is the learner's reward measuring task performance;
    \item $\reward^\ai(\state, \hstate)$ is the coach's reward,
          measuring the learner's capacity for independent performance
          (defined formally below);
    \item $\prob_0(\state, \hstate)$ is the joint initial distribution
          over states and latent skill;
    \item $\discount \in [0,1)$ is the discount factor.
\end{itemize}
\end{definition}
In a \ourgame{}, the learner is \emph{boundedly rational}~\cite{simon1990bounded}: the skill level $\hstate$ is a \emph{capability constraint} that determines the learner's policy class $\policy^\human(\cdot \mid \obsh;\, \hstate)$.
That is, although all learners seek to maximize the same reward $\reward^\human$, a beginner and an expert execute different policies to achieve a proficiency permitted by $\hstate$.
We describe a practical instantiation of this bounded-rationality model in \cref{sec:alg}.
The joint transition $\transprob(\state', \hstate' \mid \state, \hstate, \action^\human, \action^\ai)$ captures how the physical states evolve and, critically, how the coach's actions causally influence the learner's skill progression; we model the latter using a probabilistic automaton model introduced in the next section.
It is also worth noting that the game's information structure is one-sided: the only private information is the learner's latent skill $\hstate$, which is unknown to the coach but governs the learner's behavior.
The coach therefore maintains a belief $\bel_t \in \belspace_\hsset$ over $\hstate$, updated via runtime inference from observed learner behavior.
The coach's reward should depend on how well the learner would perform
\emph{alone}, which in turn depends on the learner's evolving skill $\hstate$.
We formalize this next.

\p{Learner Reward: Assisted Task Performance}
The learner's objective is a task reward collected under the combined effect of their own actions and the coach's intervention: $\reward^\human(\state_t, \action^\human_t, \action^\ai_t) = \reward^\text{task}(\state_t, \action_t)$, where $\action_t$ is the \textit{executed} action obtained by combining the learner's and coach's actions through a blending rule; we introduce a specific choice in~\eqref{eq:blending}.
This reward captures the learner's \textit{pragmatic} incentive: they are happy to lean on the coach and accept more help in exchange for better task performance. The coach, however, is graded on a different, \textit{pedagogical} exam: not how well the learner performs now with help, but how well they will perform later without it.

\p{Coach Reward: \gls{VoI}}
\gls{VoI} scores a state by the learner's expected return when they continue acting from that state \emph{without any AI assistance}:
\begin{equation}
\label{eq:voi}
    \reward^\ai(\state_t; \hstate_t)
    =
    \expectation\!\left[
        \sum_{k=0}^{\infty} \discount^k\,
        \reward^{\task}(\state_{t+k},\, \action^\human_{t+k})
        \;\middle|\;
        \begin{array}{l}
            \action^\human_{t+k} \sim \policy^\human(\cdot \mid \obsh_{t+k}; \hstate_{t+k}), \\[2pt]
            \jstate_{t+k+1} \sim \transprob(\cdot \mid \jstate_{t+k}, \action^\human_{t+k}, \action^\ai_{t+k}=0)
        \end{array}
    \right],
\end{equation}
where $\jstate := (\state,\hstate)$.
Intuitively, \gls{VoI} asks:
\textit{``If the coach were to walk away right now, how well would the learner perform?''}
A coach that maximizes \gls{VoI} is rewarded not for making the learner's performance look good with assistance, but for enabling the learner to succeed \emph{independently}.

\section{Learning to Coach: RL for Accelerating Human Skill Development}
\label{sec:alg}

In this section, we build on the \ourgame{} formalism to develop a scalable coaching policy training and deployment pipeline, which we refer to as \textbf{\gls{L2C}}.

\p{Computational Challenges}
The \ourgame{} is non-trivial to solve for three reasons. 
First, it inherits the intractability of general \glspl{POSG}~\cite{bernstein2002complexity}. 
One might hope to sidestep this by casting the \ourgame{} as an Assistance Game~\cite{fern2014decision,hadfield2016CIRL,fisac2020PPVA}, in which the one-sided private-information structure admits a reduction to a single-agent \gls{POMDP} and for which scalable solution methods have recently emerged~\cite{laidlaw2025assistancezero}. 
However, Assistance Games require an identical, \emph{shared} reward between two collaborating agents, whereas the \ourgame{} is non-cooperative by construction.
Second, the \gls{VoI} reward~\eqref{eq:voi} asks what \emph{would} happen if the coach walked away. 
Therefore, it is a \emph{counterfactual} quantity that diverges from the on-policy training trajectory and is expensive to estimate.
Third, the coach's influence on the learner's skill is \emph{indirect}:
the coach acts through shared control to modify the executed trajectory, while the learner's skill $\hstate$ updates in response to the resulting task outcomes (\eg, success or failure), not necessarily to the coach's action directly.
A standard physics simulator can roll out transition $T(s' \mid s, a)$ but says nothing about how task outcomes affect human skill change.

\p{Overview of Technical Contributions}
Our key insight is that the learner's behavior is largely governed by their skill level $\hstate$ through the skill-conditioned policy $\policy^\human(\cdot \mid \obsh; \hstate)$.
This allows us to fold the learner's policy into the environment dynamics and capture the causal effect of coaching actions on skill progression using a probabilistic finite-state automaton, effectively reducing the two-player game to a \gls{POMDP} amenable to model-free \gls{RL}.
Building on this insight, our contributions are threefold: (i) a closed-loop blending rule that adapts the assistance level $\blend$ to the coach's observations of both the physical environment and the learner behavior; (ii) a reformulation of the \ourgame{} to a single-agent \gls{POMDP}, solvable via RL, through a probabilistic model of the coach's causal influence on learner skill; and (iii) a tractable surrogate reward that is provably consistent with the counterfactual \gls{VoI} objective.

\p{Closed-Loop Blending for Adaptive Coaching}
The AI coach implements shared control by fusing the learner's action with an expert baseline via a linear blending rule:
\begin{equation}
    \label{eq:blending}
    \action = \blend \odot \policy^*(\obsr) + (\mathbf{1}-\blend) \odot \action^\human,
\end{equation}
where $\blend = \policy^\ai_\blend(\obsr) \in [0,1]^{\dim(\aset)}$ is a per-axis \emph{blending vector} that modulates the coach's assistance level, from 0 (no assistance) to 1 (full override), $\policyopt(\obsr)$ is a pretrained \textit{expert policy} that performs the task at a high level of proficiency, and $\odot$ denotes element-wise product.
Linear blending of learner and expert commands is standard in shared control~\cite{decastrodreaming,srivastava2025shared,shen2025cyber}; our formulation extends it to offer three advantages in coaching.
First, the per-dimension blending vector $\blend$ enables more flexible coaching along each independent control axis.
Second, anchoring on a pretrained expert baseline $\policyopt$ decouples task competence from coaching strategy: the coach only learns to \emph{modulate assistance} rather than synthesize full control from scratch, a deliberate design choice that we expect to reduce training difficulty compared to learning a monolithic policy end-to-end.
For the motor-skill domains we target (\eg, \gls{FPV} drone racing), expert policies are readily obtained via performance-oriented deep \gls{RL}~\cite{kaufmann2023champion,wurman2022outracing,pasumarti2025agile}.
Third, unlike prior work that modulates assistance according to a fixed rule with little to no dependence on the runtime \emph{physical} context~\cite{srivastava2025shared,shen2025cyber},
our blending factor is \emph{observation-dependent} and adapts in real time to both the physical environment and the learner's action.
To our knowledge, this is the first formulation in which the coaching strategy \textit{closes the loop} on runtime physical context, enabling spatially and behaviorally targeted interventions.

\begin{wrapfigure}{r}{0.41\textwidth}
    \centering
    \vspace{-1.2em}
    \includegraphics[width=0.41\textwidth]{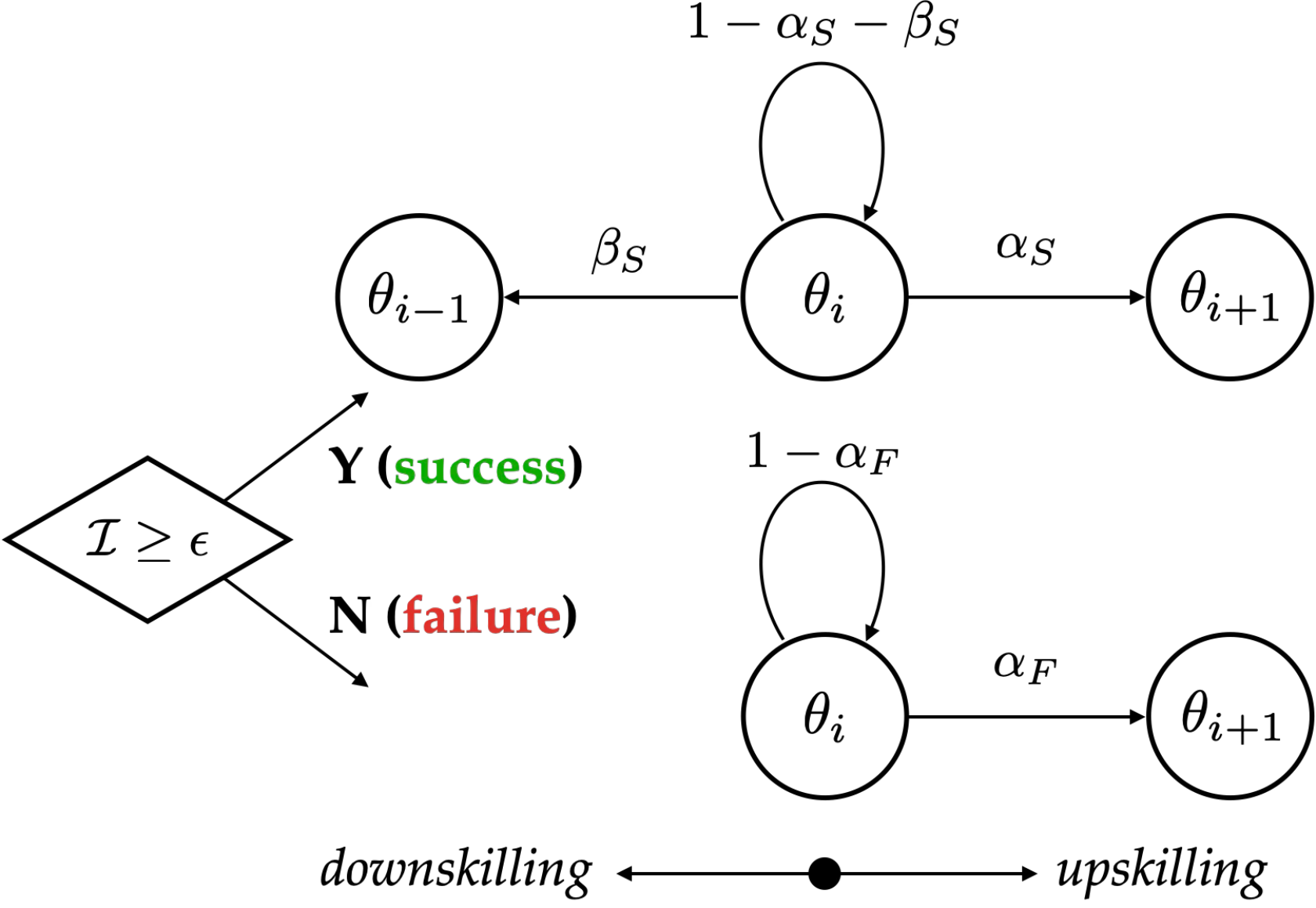}
    \vspace{-0.8em}
    \caption{Hybrid PFA for simulating skill change triggered by success or failure events.}
    \label{fig:PFA}
    \vspace{1em}
\end{wrapfigure}

\p{Simulating Human Skill Change Due to Coaching}
To train a coaching policy, we augment the robot's transition dynamics
with two learner-side components: a skill-conditioned control policy, and a model of how the learner's latent skill $\hstate$ evolves in response to coaching events such as successful or failed task attempts, effectively turning the \ourgame{} into a single-agent \gls{POMDP}. 
Given an expert value function $\qfuncopt(\obsh, \action)$, we model learners of different skill as Boltzmann noisily-rational~\cite{luce1959individual} decision-makers whose action variance shrinks with proficiency:
\begin{equation}
    \label{eq:boltzmann}
    \policy^\human(\action^\human \mid \obsh; \hstate) \propto \exp\!\left( \beta(\hstate) \cdot \qfuncopt(\obsh, \action^\human) \right),
\end{equation}
where the rationality coefficient $\beta(\hstate)$ increases monotonically in $\hstate$:
a simulated novice learner samples actions near-uniformly over the action set, while more skilled learners concentrate their actions more sharply around $\action^* = \argmax_\action \qfuncopt(\obsh, \action)$.
For a bounded action set $\aset$, 
we use the standard Laplace expansion~\cite{bishop2006PRML} to approximate~\eqref{eq:boltzmann} by a Gaussian centered at $\action^*$ and truncated to $\aset$.
Our framework is designed to be \emph{modular} with respect to the learner model: it only requires that the learner be modeled by a $\hstate$-parameterized policy class $\policy^\human(\cdot \mid \obsh; \hstate)$. 
Accurately modeling skill-conditioned learner behaviors remains an open problem, and we want our framework to benefit from future progress on it without having to redesign the algorithmic framework.
The Boltzmann model~\cref{eq:boltzmann} used in our experiments is one tractable choice that is simple and interpretable.
Many alternatives could be plugged into the same training pipeline, including behavior cloning~\cite{gopinath2024computational}, fictitious co-play~\cite{decastrodreaming}, dynamic games~\cite{fisac2020PPVA,mazumdar2025tractable}, inverse games~\cite{liu2023learning,hu2025think}, and mix of policies~\cite{jacob2022modeling,laidlaw2025assistancezero}.

Next, we need a model that captures the coach's causal influence on the learner's skill evolution through task outcomes.
To this end, we propose to model the learner's skill change as \emph{probabilistic transitions} triggered by success and failure events, similar to the human-adaptation framework of Nikolaidis~et~al.~\cite{nikolaidis2017human}.
We leverage a hybrid \gls{PFA} model (\cref{fig:PFA}) that operates in two modes governed by the \textit{performance index} $\pindex_t(\history_t)$, a measure of human learning progress, where $\history_t = (\state_t, \action^\human_t, \state_{t-1}, \action^\human_{t-1}, \ldots)$ is the history of states and human actions.
A natural choice is the cumulative task reward over a horizon: $\pindex_t(\history_t) = \sum_{k=0}^{T-1} \reward^{\mathrm{task}}(\state_{t-k}, \action^\human_{t-k})$.
The learner is deemed to have \emph{failed} if $\pindex_t < \epsilon$ for a prescribed threshold $\epsilon$, set semantically (e.g., causing a collision) or calibrated from a pilot study.
When the learner succeeds ($\pindex_t \ge \epsilon$), they upskill with probability $\alphaS(\hstate)$, downskill with probability $\betaS(\hstate)$, or remain unchanged otherwise. When the learner fails ($\pindex_t < \epsilon$), they upskill with probability $\alphaF(\hstate)$.
Each rate is governed by a sigmoid in $\hstate$ parameterized by a slope and offset with pedagogically meaningful trends: $\alphaS(\hstate)$ is highest for novices (small wins matter most early on), $\betaS(\hstate)$ decreases with $\hstate$ (skill becomes harder to lose once consolidated), and $\alphaF(\hstate)$ grows with $\hstate$ (skilled learners extract more from their mistakes). 
To cover a diverse learner population, we randomize the sigmoid parameters at the start of each training episode to generalize the coaching policy across learners with different learning capabilities.

\p{Training a Coaching Policy with Model-Free Reinforcement Learning}
We train the coaching policy $\policy^\ai_\blend$ using proximal policy optimization (PPO)~\cite{schulman2017ppo}, leveraging privileged access to the learner's skill level $\hstate$ in simulation. The skill level is included in the observation, and we use a tractable surrogate reward $\reward^\ai_t = \hstate_t - \hstate_{t-1}$ in place of the counterfactual \gls{VoI} reward~\eqref{eq:voi}. At deployment, 
the coach maintains an estimate $\hat{\hstate}$ based on belief $\bel_t$ recursively updated using runtime inference.
We now show that optimizing skill directly is sufficient for optimizing \gls{VoI}. Note that maximizing cumulative skill improvement can help reduce variance and, via telescoping, is equivalent to maximizing skill levels up to a positive affine transformation.
For any physical state $\state \in \sset$ and skill level $\hstate \in \hsset$, define the learner's \emph{independent continuation value} as
$\valfunc^{\ind}(\state,\hstate)
\,\triangleq\,
\expectation\!\left[
    \sum_{k=0}^{\infty}\discount^k\,\reward^\task(\state_{t+k},\action^\human_{t+k})
    \;\middle|\;
    \begin{array}{l}
        \state_t=\state,\ \hstate_t=\hstate,\action^\ai_{t+k}\equiv 0,\ \forall k\ge 0
    \end{array}
\right]
$.
Fixing a policy-independent evaluation distribution $\evaldist \in \Delta(\sset)$, we define the \emph{evaluation-averaged \gls{VoI}} as
$
\bar{\reward}^\ai(\hstate)
\,\triangleq\,
\expectation_{\state\sim\evaldist}\!\left[\valfunc^{\ind}(\state,\hstate)\right]
$.
The following result provides a formal bridge between the \emph{training-time} use of privileged skill information and the \emph{deployment-time} objective of maximizing the learner's independent performance; the proof can be found in \autoref{app:proof_main}.

\begin{proposition}[Skill dominance implies \gls{VoI} dominance]
\label{prop:skill_to_voi_equivalence_main}
Suppose $\bar{\reward}^\ai(\hstate)$ is nondecreasing in skill $\hstate$. If, at every skill level, policy $\policy^\ai$ makes the learner at least as likely to upskill (i.e., increase $\theta$) in the next step as $\tilde{\policy}^\ai$ does, then
$
\expectation_{\policy^\ai}\!\left[
    \sum_{t=0}^{\infty}\discount^t\,\bar{\reward}^\ai(\hstate_t)
\right]
\ge
\expectation_{\tilde{\policy}^\ai}\!\left[
    \sum_{t=0}^{\infty}\discount^t\,\bar{\reward}^\ai(\hstate_t)
\right].
$
\end{proposition}
\cref{prop:skill_to_voi_equivalence_main} shows that, if a coaching policy ${\policy}^\ai$ is able to increase the learner's upskill probability, it also dominates on the evaluation-averaged \gls{VoI} return $\expectation_{{\policy}^\ai}\!\left[\sum_{t=0}^\infty \discount^t \bar{\reward}^\ai(\hstate_t)\right]$.
The state-averaging through $\evaldist$ is necessary because different coaching policies may induce different physical state trajectories; $\bar{\reward}^\ai(\hstate)$ therefore decouples the coach's \gls{VoI} objective from the policy-induced state distribution.
For example, in our \gls{FPV} drone racing deployment to be introduced next, $\evaldist$ corresponds to the distribution of starting states in unassisted evaluation laps.

The training simulator comprises four modules (see \autoref{fig:front} left): a \emph{physics simulator} that advances physical states given the blended human and AI actions, a \emph{skill dynamics module} that evolves $\hstate$ via the hybrid \gls{PFA} (\cref{fig:PFA}), a \emph{learner behavior module} that samples $\action^\human$ from the skill-conditioned Boltzmann policy~\eqref{eq:boltzmann}, and an expert policy $\policyopt(\obsr)$.
Each training episode terminates when an upskill event occurs.
The training procedure is summarized in \cref{alg:training}.

\p{Deployment: A Case Study in Coaching FPV Drone Racing}
We apply \gls{L2C} to \gls{FPV} drone racing~\cite{kaufmann2023champion} as an illustrative example, where the learner pilots a quadrotor through a sequence of gates via a remote controller (\autoref{fig:front} right).
At deployment, the AI coach operates through three key components.
\emph{(i) Assistance modalities:} The coach executes shared control through the blending rule~\eqref{eq:blending}, and augments it with multimodal feedback channels such as verbal instructions and visual cues indicating the error between the expert and the learner's current control.
\emph{(ii) Runtime skill inference:} Since $\hstate$ is unobservable at deployment, the coach maintains a belief $\bel_t$ over $\hstate$, which yields an online estimate $\hat{\hstate}$ and is updated via Bayesian inference
$\bel^+(\hstate) \propto \prob(\obsr \mid \hstate) \bel(\hstate)$.
In our drone racing setting, 
we partition the track into per-gate segments and maintain an independent belief $\bel_t^\ell$ over the learner's skill at each gate $\ell$; the per-gate likelihood $\prob(\obsr \mid \hstate^\ell)$ is a Gaussian on the realized segment passage time, centered at a skill-dependent target time $\tau^\ell_{\text{tgt}}(\hstate^\ell)$ calibrated from a pilot study.
Under this likelihood model, faster passages shift belief mass toward higher skill, slower passages toward lower.
Alternative skill inference schemes include Bayesian Knowledge Tracing~\cite{corbett1994knowledge,piech2015deep}, which typically requires offline model fitting from massive human data.
\emph{(iii) Structured training loop:} The coach provides continuous assistance via~\eqref{eq:blending} in the normal mode; upon failure ($\pindex < \epsilon$), it resets the environment and asks the learner to re-attempt the failed segment under guided practice.
In our drone racing setting, the reset position is one gate before the failed segment, giving the learner room to re-enter the segment with appropriate momentum.
Design details of the coaching system can be found in Appendix~\ref{app:deploy}.

\section{Experimental Results}

\label{sec:experimental-results}

In this section, we present experimental results demonstrating the effectiveness of our \gls{L2C} coach and its superiority over state-of-the-art baselines.
We conducted a user study with $N=33$ participants from diverse drone-flying backgrounds in a high-fidelity \gls{FPV} drone racing simulator  that we built on top of Pasumarti et al.~\cite{pasumarti2025agile} and Isaac Lab~\cite{mittal2023orbit}.
We structure each trial into three phases: a pre- and post-coaching test (2 laps each), and a main coached training session.
To prevent participant fatigue, we limit the training session to 15 laps (around 40 minutes) and restrict the human's control authority to yaw and roll rate while automating the pitch rate and thrust.
Simulation and user study details can be found in \cref{app:training:sim} and \cref{app:human_study}, respectively.

\begin{figure}[!t]
    \centering
    \includegraphics[width=\linewidth]{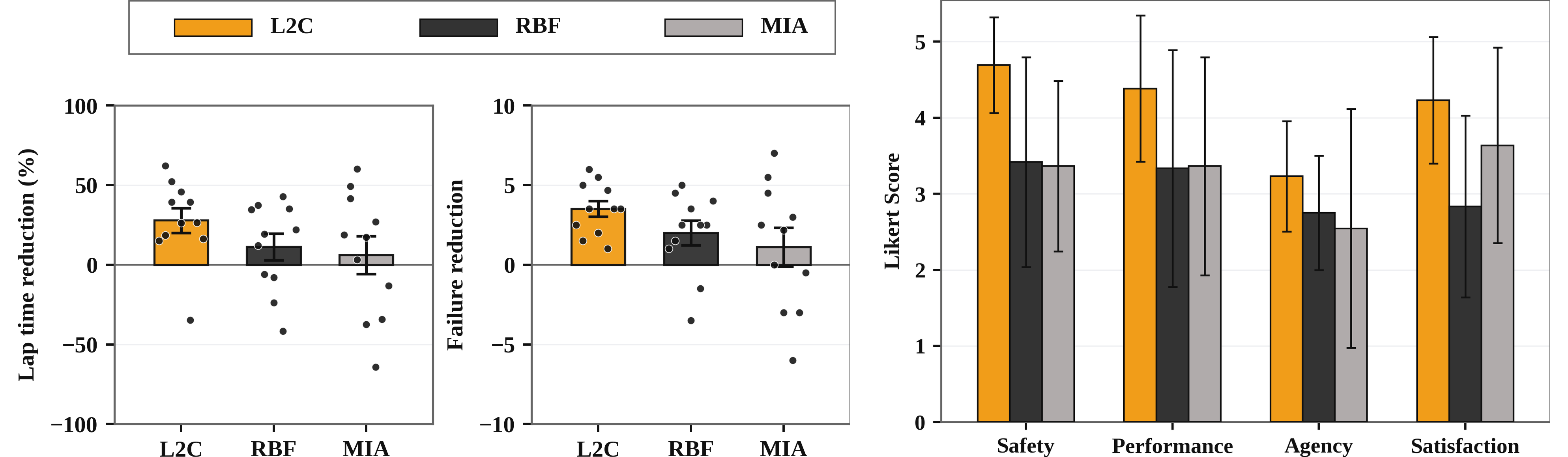}
    \caption{After coaching, learners trained with our \gls{L2C} coach show significant reductions in lap time and failure count (left) and report higher safety, performance, agency, and satisfaction than both baselines (right).
    }
    \vspace{-1.0em}
    \label{fig:experimental-results}
\end{figure}

\p{Baselines and Metrics}
Each participant was randomly assigned to one of three AI coaches: \gls{L2C} and two baselines: \gls{RBF}~\cite{shen2025cyber}---a recently proposed AI coaching scheme that gradually fades the assistance level over the training session following a fixed curve and modulation rule agnostic to learner skill and physical context, and \gls{MIA}---an RL-based AI copilot, following human-centered safety filters~\cite{oh2025safety}, that minimally modifies the learner action to achieve safety and task completion.
To ensure a fair comparison, \gls{L2C} and \gls{MIA} use the same expert policy $\policy^*$; \gls{MIA} is calibrated so that its behavior in absence of human input matches that of $\policy^*$.
All three methods use the same language instructions and visual cues as additional coaching modalities (see details in \cref{app:deploy}).
Based on self-reported drone-operating experience measured on a five-point scale, we found no statistically significant differences in initial skill levels across the three groups.
We measure human skill change with two metrics: lap time and failure count (collisions and timeouts; a gate must be passed within 8 seconds).
We further collect subjective evaluations of the coaching experience via a post-trial questionnaire.

\p{Hypotheses}
We structure our main results around the following hypotheses:
\begin{itemize}[label={}, topsep=0pt, itemsep=1pt, leftmargin=*]
    \item \textbf{H1:} \gls{L2C} reliably improves the learner's performance in both lap time and failure count.
    
    \item \textbf{H2:} Neither baseline replicates \gls{L2C}'s dual improvement.
    
    \item \textbf{H3:} \gls{L2C} produces larger mean improvements than each baseline on both outcomes.
\end{itemize}

\p{Results}
For \gls{L2C}, a paired-samples $t$-test comparing pre- and post-coaching lap times yielded a mean reduction of 27.9\% with $p = 0.005$ and Cohen's $d_z = -1.08$ (large).
Similarly, a paired $t$-test on per-lap failure count yielded a mean reduction of 3.52 failures per lap with $p < 0.001$ and Cohen's $d_z = -2.13$ (very large).
For \gls{RBF}, paired tests showed no reliable change in either lap time (mean reduction 11.3\% with $p = 0.21$ and $d_z = -0.41$), but the paired test on failure count showed significance (mean reduction 2 failures per lap with $p = 0.027$, $d_z = -0.78$) at roughly one-third the magnitude of \gls{L2C}'s effect.
For \gls{MIA}, paired tests showed no reliable change in either lap time (mean reduction 6.2\% with $p = 0.62$ and $d_z = -0.15$) or failure count (mean reduction 1.11 failures per lap with $p = 0.38$ and $d_z = -0.28$).
Those results validate \textbf{H1} and \textbf{H2}.

\begin{figure}[!t]
    \centering
    \includegraphics[width=\linewidth]{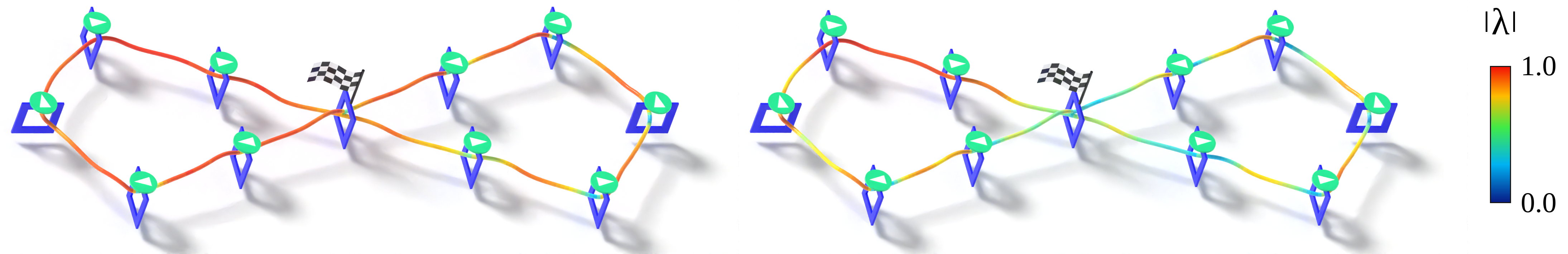}
    \caption{\gls{L2C} adjusts assistance based on estimated skill $\hat{\hstate}$ and physical context. \emph{Left:} $\hat{\hstate}=0$. \emph{Right:} $\hat{\hstate}=0.5$.
    }
    \vspace{-1em}
    \label{fig:l2c-alpha-visualization}
\end{figure}

We further ran Welch's two-sample $t$-tests with Holm correction applied within each outcome.
On lap time, L2C outperformed MIA ($\Delta = -21.7$ \%, Hedges' $g = -0.63$) and RBF ($\Delta = -16.6$ \%, $g = -0.60$); on failure count, L2C outperformed MIA ($\Delta = -2.41$ failures per lap, $g = -0.76$) and RBF ($\Delta = -1.52$ failures per lap, $g = -0.67$).
All four contrasts favor \gls{L2C} with medium-to-large effect sizes between 0.60 and 0.76 and $p$-values between 0.09 and 0.16 (expected given the sample-size limit at $n = 11$ per group; \gls{L2C} is also the only method to reach within-subject significance on both outcomes).
This result supports \textbf{H3}.

The objective metrics in \autoref{fig:experimental-results} show that \gls{L2C} achieves larger reductions in lap time and failure count than both baselines.
On the post-trial questionnaire, \gls{L2C} also received higher subjective ratings than both baselines on \textit{safety} (``Compared to my first race, I flew more safely in the final race.'') and \textit{performance} (``I believe my overall drone racing performance has improved since my first race.'') than both baselines, consistent with its objective improvements.
Higher \textit{agency} (reverse-coded from ``The AI coach intervened too much throughout the race.'') and \textit{satisfaction} (``Overall, I am happy with how the training went.'') ratings further suggest that adaptive coaching preserves user control and improves the perceived training experience.
\autoref{fig:human_traj} provides qualitative examples of how human trajectories change after coaching, complementing aggregate metrics in \autoref{fig:experimental-results}.
Finally, in \autoref{fig:l2c-alpha-visualization}, we inspect how \gls{L2C} modulates assistance based on the estimated learner skill $\hat{\hstate} \in [0,1]$ and physical context.
The magnitude of blending vector $\blend = \policy^\ai_\blend(\obsr)$ along example trajectories shows two emerging forms of adaptation: 
across learners, a novice ($\hat{\hstate}=0$, left) receives more assistance than a more skilled learner ($\hat{\hstate}=0.5$, right); 
within each trajectory, $|\blend|$ varies sharply around the gates, the task-critical states that determine whether the learner collides or passes successfully.

\section{Conclusions}

\p{Limitations and Future Work}
While our results demonstrate the promise of learning-based AI coaching, we acknowledge several limitations and outline directions for future work.
Although our \gls{PFA} appeared to be sufficient for modeling learner behaviors in drone racing, the simulated learners may not capture the full variability of real human learning in more complex motor-skill training settings.
Future work could explore data-driven alternatives such as behavior cloning~\cite{gopinath2024computational}, world models~\cite{decastrodreaming}, or mix of policies~\cite{jacob2022modeling,laidlaw2025assistancezero}.
Our current system learns only the assistance modulation of physical control, while verbal and visual cues are predefined and triggered by fixed rules.
Future work could jointly optimize a multimodal coaching policy with a foundation-model backbone.

\p{Summary}
As AI systems grow more capable, the question is no longer whether they can replace human skill, but how they can help build it.
We have argued, and shown empirically, that optimizing for independent human competence gives both a principled definition of human-empowering AI and a practical recipe for building one.
Our game-informed RL framework produces a coach that measurably accelerates human skill development in a comprehensive drone racing user study. The same principle matters far beyond motor-skill training.
Coding agents, in particular, are currently optimized for task performance, with no explicit incentive for what the human retains.
Our framework offers a concrete alternative: AI systems that not only assist their users but also help them grow.

\clearpage

\acknowledgments{The authors thank Hongrui Zheng and Helen Loeb for insightful discussions.}

\bibliography{references}  %

\appendix

\section{Coaching Policy Training Details}
\label{app:training}

\subsection{Drone Racing Simulation}
\label{app:training:sim}

We implement the FPV drone-racing task as a vectorized direct-RL environment in Isaac Lab. Our simulator and low-level quadrotor dynamics follow the agile drone-racing simulation setup of Pasumarti et al.~\cite{pasumarti2025agile}, adapted to our coaching task. The simulator contains a small quadrotor, a static ground plane, and a sequence of collidable gates loaded from USD assets. In training, the scene is replicated across many parallel environments; in the default configuration we use 4096 environments. The task objective is to pass a fixed sequence of square gates quickly while avoiding crashes, altitude violations, and gate timeouts.

\begin{figure}[hbt]
    \centering
    \includegraphics[width=\linewidth]{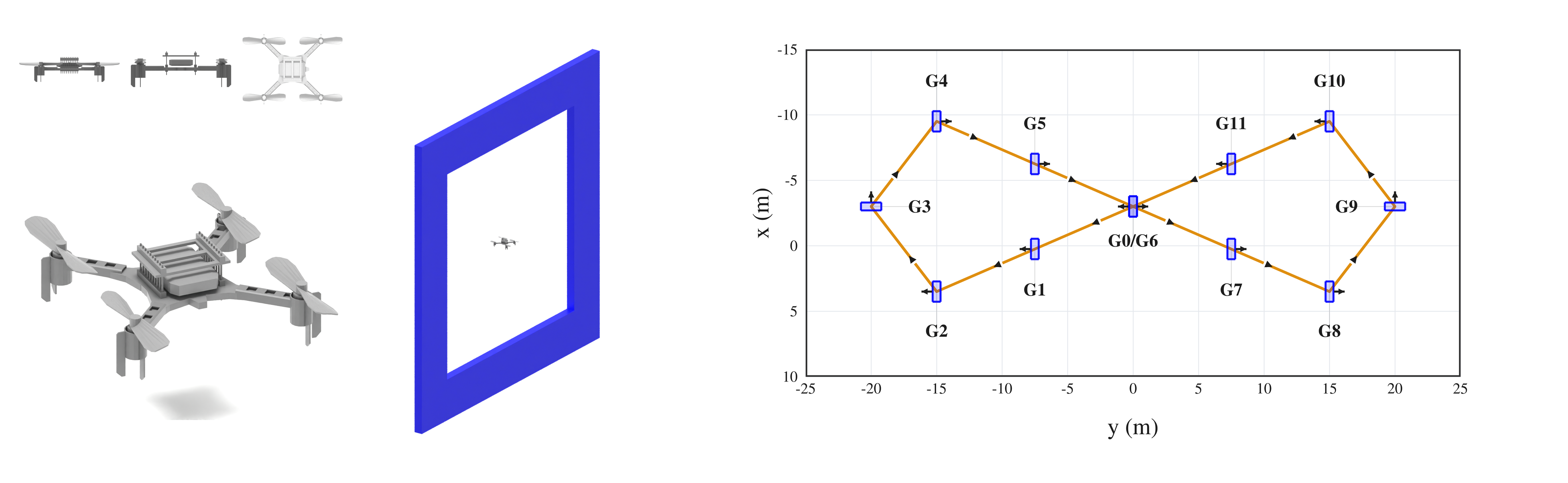}
    \caption{Drone-racing simulation overview. \emph{Left:} visualization of the quadrotor physical model used. \emph{Middle:} true-scale size illustration of the quadrotor relative to the gate opening. \emph{Right:} track-layout visualization showing gate order, positions, headings, and traversal direction.}
    \label{fig:simulation-overview}
\end{figure}

We use the same quadrotor dynamics as in~\cite{pasumarti2025agile}. The high-level policy runs at 50 Hz and outputs collective thrust and desired body-rate commands. Isaac Lab physics and the body-rate PID loop run at 250 Hz, so each policy action is held for five physics steps. The body-rate controller maps commands to a wrench, which is allocated to motors before thrust, torques, and aerodynamic drag are applied to the rigid body.

\begin{itemize}[leftmargin=*]
    \item \textbf{Track and objective.} The default \texttt{figure8flat} course contains 12 gates arranged as upper/lower loops around a shared center crossing. Gates are passed in a fixed cyclic order, and the learner is evaluated by lap completion time and failure events.
    \item \textbf{Gate geometry and passing.} Each gate has a 1 m square opening and is placed from a waypoint pose $[x,y,z,\phi,\theta,\psi]$. For the \texttt{figure8flat} course, the gate $xy$ positions and headings are visualized in the right track-layout panel of \autoref{fig:simulation-overview}, and all gates use the fixed height $z=2.0\,\mathrm{m}$. A gate is counted as passed only when the line segment from the previous to the current drone position intersects the current gate plane inside the square opening and crosses in the correct direction.
    \item \textbf{Observation.} The policy observation combines drone state and task-relative state: body angular velocity, global position, body-frame linear velocity, attitude quaternion, the drone position expressed in the current gate frame, and the current gate heading.
    \item \textbf{Action and frequency.} The policy action is a four-dimensional normalized command $\mathbf{a}=[a_0,a_1,a_2,a_3]\in[-1,1]^4$, corresponding to collective thrust and desired roll, pitch, and yaw rates. The action is evaluated at 50 Hz and held for five 250 Hz physics/PID steps.
    \item \textbf{Termination and reset.} Episodes reset on contact/crash, altitude above 6.0 m, altitude below 0.1 m after a 1.5 s ground-contact grace period, or failure to reach the next gate within 7.0 s.
\end{itemize}

\subsection{Network Architecture}
\label{app:training:nn}

We trained two policies using PPO: an expert drone-racing policy $\policy^*$ and an L2C coaching policy $\policy^\ai_\blend$. Both use separate actor and critic multilayer perceptrons with ELU activations. The actor hidden dimensions are $[128,128]$, and the critic hidden dimensions are $[512,256,128,128]$. Each actor outputs the mean of a diagonal Gaussian policy with a learned per-action standard deviation initialized to 1.0; at evaluation time we use the actor mean.

The expert policy receives a 17-dimensional observation consisting of body angular velocity, world position, body-frame linear velocity, world-frame orientation quaternion, drone position relative to the current gate in the gate frame, and current gate yaw. Empirical observation normalization is disabled. The actor of $\policy^*$ outputs a four-dimensional normalized command comprising collective thrust and roll, pitch, and yaw rates.

The L2C coaching policy $\policy^\ai_\blend$ receives the same 17-dimensional racing observation concatenated with the learner level $\theta$, giving an 18-dimensional coach observation. The coach actor outputs a two-dimensional blending vector $\lambda$ for roll and yaw assistance modulation. Because participants control only roll and yaw in our human study, $\lambda$ is instantiated only on these two human-controlled axes; thrust and pitch are handled by the
expert policy
and thereby not coached.

\subsection{PPO Training}
\label{app:training:ppo}

We train both policies with PPO in Isaac Lab. First, we train an expert policy capable of racing the drone at high competence. Second, we freeze this expert policy and train the L2C coach. \autoref{tab:ppo-hyperparameters} reports the PPO settings for both training stages; italic entries indicate settings that differ between the two PPO runs. The coaching policy training loop is summarized in \cref{alg:training}.

\begin{table}[!t]
    \centering
    \caption{PPO hyperparameters for the expert policy and the L2C coach policy. Italic entries indicate settings that differ between the two PPO runs.}
    \label{tab:ppo-hyperparameters}
    \small
    \setlength{\tabcolsep}{6pt}
    \renewcommand{\arraystretch}{1.28}
    \begin{tabularx}{\linewidth}{@{}>{\raggedright\arraybackslash}p{3.0cm}>{\raggedright\arraybackslash}X@{}}
        \toprule
        \textbf{Parameter} & \textbf{Setting} \\
        \midrule
        Learned policy & \emph{Expert: $\policy^*$}\newline \emph{Coach: $\policy^\ai_\blend$}\newline The expert policy is separately trained and frozen during coaching policy training \\
        \addlinespace[0.25em]
        Observation / action & \emph{Expert $\policy^*$: 17-D observation, 4-D action (thrust, roll rate, pitch rate, and yaw rate)}\newline \emph{Coach $\policy^\ai_\blend$: 18-D observation, 2-D blending vector $\lambda$}
        \\
        \addlinespace[0.25em]
        Parallel environments & 4096 \\
        \addlinespace[0.15em]
        Rollout length & 24 policy steps per environment (0.48 s at 50 Hz) \\
        \addlinespace[0.15em]
        Optimization & 4 minibatches, 5 epochs \\
        \addlinespace[0.15em]
        Learning rate & $5\times10^{-4}$ with adaptive KL schedule; target KL $= 0.01$ \\
        \addlinespace[0.15em]
        Discount / GAE & $\gamma=0.99$, $\lambda=0.95$ \\
        \addlinespace[0.15em]
        PPO objective & clip $= 0.2$; clipped value loss; value-loss coefficient $= 1.0$ \\
        \addlinespace[0.15em]
        Regularization & max gradient norm $= 1.0$\newline \emph{Expert-policy entropy $= 0.0$}\newline \emph{Coach-policy entropy $= 0.1$} \\
        \addlinespace[0.25em]
        Reward signal & \emph{Expert policy: racing task and speed-regularization rewards}\newline \emph{Coach policy: skill-level difference}\\
        \addlinespace[0.25em]
        Training horizon & Up to 10,000 PPO iterations \\
        \bottomrule
    \end{tabularx}
\end{table}

\begin{algorithm}[!t]
\caption{Coaching Policy Training}
\begin{algorithmic}[1]
\label{alg:training}

\STATE \textbf{Input:} Expert policy $\policy^*$, \gls{PFA}, human policy parameterization~\eqref{eq:boltzmann}, physics simulator

\STATE Initialize a set of parallel environments, each with an initial skill level $\hstate \sim \mathcal{U}([0,1])$ 

\FOR{iteration $= 1, \dots, N_{\text{rollout}}$}

    \FOR{step $= 1, \dots, N_{\text{steps}}$}
        \FOR{each environment}

            \STATE Observe $\obsh_t$, $\obsr_t$, and skill level $\hstate_t$ 

            \STATE Sample human action via~\eqref{eq:boltzmann}:
            $
            a^\human_t \sim \pi^\human(\cdot \mid \obsh_t; \hstate_t)
            $

            \STATE Sample blending vector:
            $
            \lambda_t \sim \policy^\ai_\blend(\obsr_t)
            $

            \STATE Compute shared control via~\eqref{eq:blending}:
            $
            a_t = \lambda_t \odot \pi^*(\obsr_t) + (\mathbf{1} - \lambda_t) \odot a^\human_t
            $

            \STATE Step the environment with the physics simulator and \gls{PFA}, and observe $s_{t+1}$ and $\hstate_{t+1}$

            \STATE Store transition
            $
            (\obsr_t, \lambda_t, \hstate_{t+1}, s_{t+1}, d_t)
            $, where the done flag $d_t$ is set to True and the environment resets if the skill level changes, i.e., $\hstate_{t+1} \neq \hstate_t$; upon reset, the environment resamples $\hstate$ from $\mathcal{U}([0,1])$.  
            
        \ENDFOR
    \ENDFOR

    \STATE Update $\policy^\ai_\blend$ using PPO~\cite{schulman2017ppo} with stored transitions across all parallel environments

\ENDFOR

\end{algorithmic}
\end{algorithm}

\paragraph{Expert policy training.}
The expert policy is trained with both task rewards and auxiliary regularization rewards. The task rewards encourage progress toward the next gate and successful gate passage. The auxiliary terms encourage the drone to face the gate and avoid motion that novices cannot follow, such as excessive speed, tilt, and yaw rate. These auxiliary terms are important for the human-study setting: with only sparse task success rewards, the learned policy can race quickly but often uses non-intuitive headings and turns that are difficult for a human learner to parse and imitate. We therefore guide the expert policy toward a stable, gate-facing flight style that remains efficient but is easier for a human to learn.
The expert policy uses entropy coefficient 0.0 and takes approximately 8 hours to train.
We list and explain important reward components below.
\begin{itemize}[leftmargin=*]
    \item \textbf{Progress reward.} The expert policy receives $5.0\,\Delta(\rho+n_{\mathrm{gate}})$, where $\rho=1-\tanh(\|p_{\mathrm{gate}}-p_{\mathrm{drone}}\|/3)$ and $n_{\mathrm{gate}}$ is the number of passed gates. This delta form rewards forward progress through the cyclic gate sequence.
    \item \textbf{Gate reward.} The expert policy receives $50.0 \cdot \mathbf{1}_{\mathrm{gate}}$ when the drone crosses the current gate plane inside the square opening and in the correct direction.
    \item \textbf{Failure cost.} On a terminated reset, the scalar racing reward is replaced by $-200.0$. This covers crash/contact, altitude violation, and gate-timeout termination. The crash and time terms are logged, but their active coefficients are 0.0 in this run.
    \item \textbf{Neutral terms.} Crash and time terms are tracked as diagnostics with active coefficients 0.0; failure affects the scalar reward through the terminal cost above.
    \item \textbf{Auxiliary terms.} The expert policy training also tracks gate-facing alignment, speed, tilt, and yaw rate reward components with coefficients $0.5$, $-0.015$, $-0.01$, and $-0.01$, respectively. 
\end{itemize}

\paragraph{L2C coach training.}
After training the expert policy, we freeze its weights and reuse it as a fixed corrective reference during L2C coach training. At each timestep, the coach environment samples a simulated learner with skill level $\hstate$, the learner produces an action via the Boltzmann policy in~\eqref{eq:boltzmann}, and the coach outputs a blending vector $\blend$ that combines the expert and learner actions according to the rule in~\eqref{eq:blending}. The resulting executed action drives the physics simulator, and the coach is rewarded based on the change in $\hstate$ governed by the \gls{PFA}. We train the coaching policy with PPO using an entropy coefficient of 0.1 to encourage exploration over the blending space, with the full training run taking approximately 6.5 hours on a single RTX 5080 GPU.

\begin{figure}[t]
    \centering
    \includegraphics[width=\linewidth]{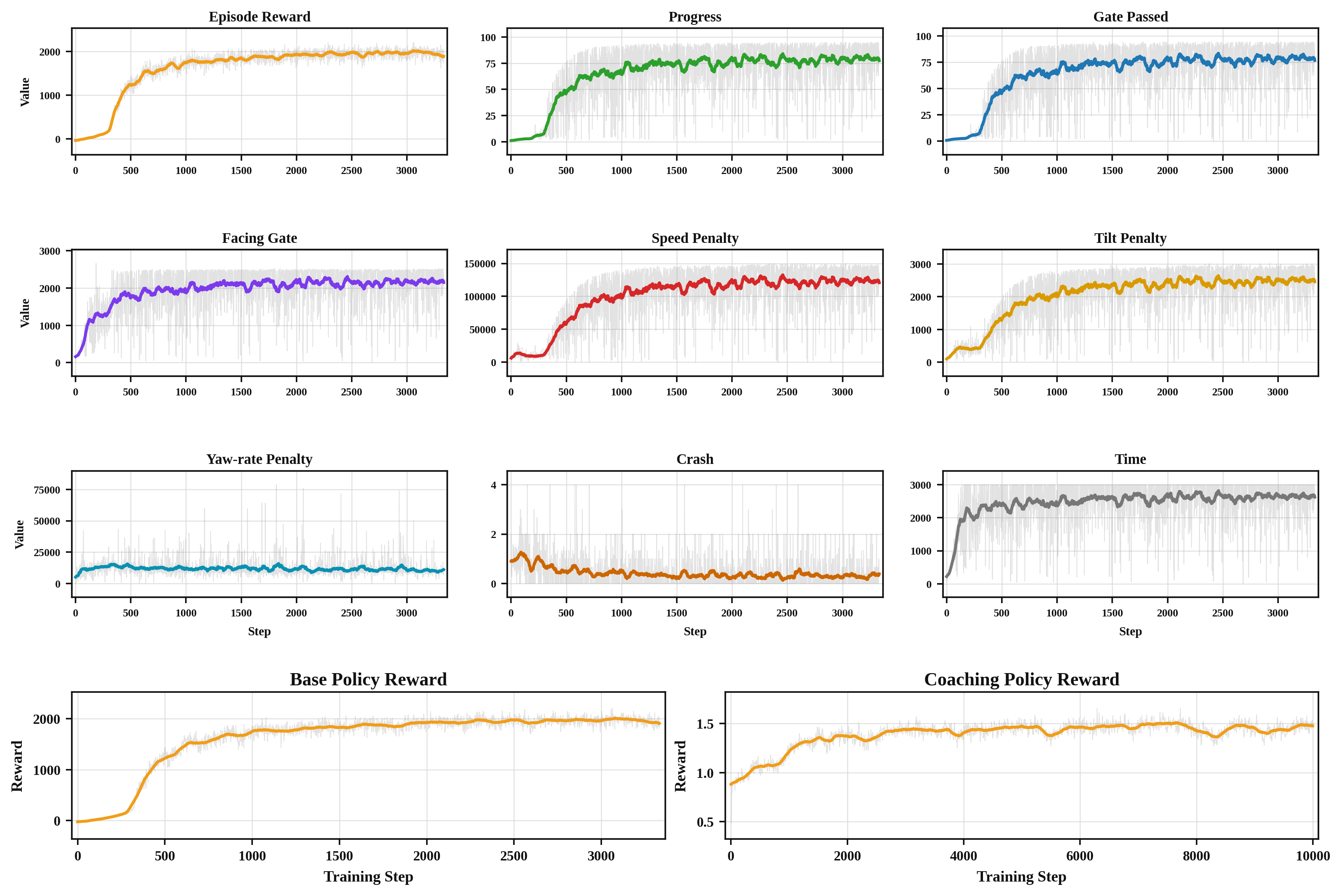}
    \caption{PPO training curves. The top blocks show expert policy reward components, and the bottom row shows the total expert policy and L2C coaching policy training curves. Reward values are normalized internally by Isaac Lab.}
    \label{fig:ppo-reward-panel}
\end{figure}

\subsection{Training Hardware Details}
\label{app:deploy:hardware}
We use a GPU-accelerated Isaac Lab simulation for policy training and user-study deployment. The final experiment stack uses Python 3.10, PyTorch 2.7.0, Isaac Sim 5.1.0, and Isaac Lab 0.48.5. Training the expert policy on a single NVIDIA GeForce RTX 5080 GPU takes approximately 8 hours, and training the L2C coach takes approximately 6.5 hours.

During the human study, all runtime components, including the simulator, policy inference, graphics rendering, controller input, and verbal feedback, ran locally on a single workstation. Expert and L2C policy weights were loaded from PyTorch checkpoints and executed in the same software stack used during training, keeping the deployment self-contained and reproducible.

\section{Deployment Details}
\label{app:deploy}

As shown in \cref{fig:front}, the experiment employs a high-fidelity FPV drone racing simulator for coaching a human pilot to fly a quadrotor through a sequence of gates.
The human participant will operate the drone using a remote controller.
The task will be presented on a screen showing a first-person onboard camera view, matching the visual feedback available in real-world FPV drone racing~\cite{kaufmann2023champion}.

The simulator is implemented in Isaac Lab~\cite{mittal2023orbit} and uses the same environment as the one in policy training.
The system logs the human action, expert action, blended action, drone state, gate passage events, failures, lap or segment times, and assistance magnitude during each trial.

All three study conditions use the same simulator, track, user interface, visual cues, and verbal feedback. The two baselines follow their original implementations and hyperparameters with minimal adaptations required for moving from the car racing domain to drone racing with the two-axis human-control interface.
In summary, the comparison keeps the deployment environment fixed and changes only the assistance method.

The simulated racing environment (\cref{fig:human_traj}) consists of 12 gates in a figure-eight layout (\cref{fig:simulation-overview}), with two gates spatially overlapping at the central crossing. Participants are instructed to fly through the gates in the prescribed order, completing the course as quickly as possible.

\subsection{User Interface and Visual Cues}
We show the user interface presented to participants during the FPV drone racing study in \cref{fig:UI}, rendered in real time in Isaac Lab.
A static drone rotor is rendered to reduce visual clutter.
Top-left displays a color-coded sequence of passed and upcoming gates, with white indicating the current target gate, yellow/green indicating that the gate passage time is slower/faster than that of the last lap, and purple indicating the fastest gate passage time so far.  
The green arrow on the top indicates the relative direction of the target gate (highlighted in yellow) to the drone.
The bottom bars are visual cues, shared across all coaching methods, that aid the human's steering control, where the yellow line segment shows the expert command, temporally smoothed by blending the current and previous expert actions to reduce visual jitter, and the green block shows the learner's action. 
A head-up display (HUD) on the bottom displays the current lap time, current gate index, current lap number, last lap time, and best lap time.

\begin{figure}[H]
    \centering
    \includegraphics[width=0.75\linewidth]{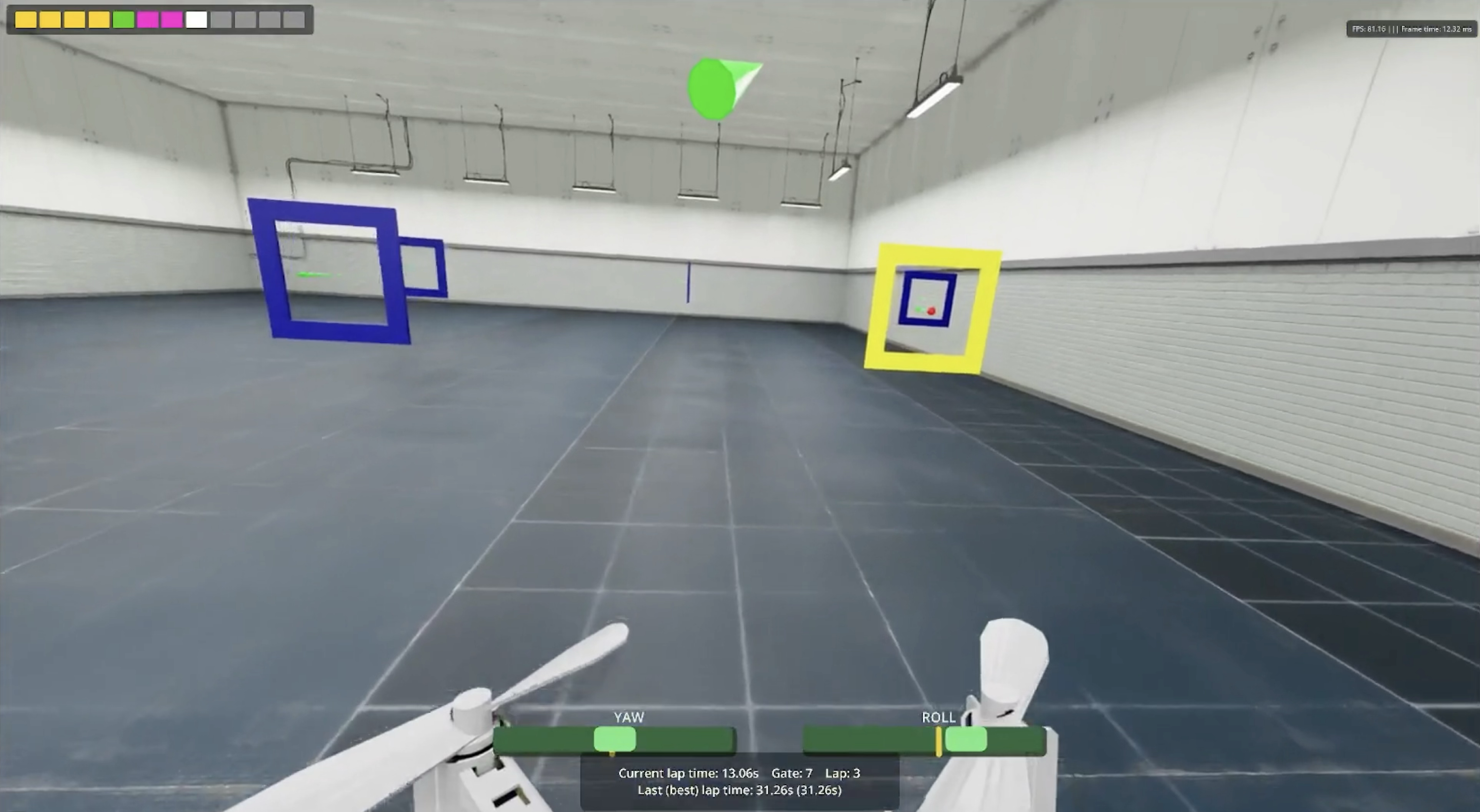}
    \caption{User interface presented to participants during the AI coaching for FPV drone racing study.
    }
    \label{fig:UI}
\end{figure}

\subsection{Verbal Feedback}
All coaches (\gls{L2C}, \gls{RBF}, and \gls{MIA}) use verbal feedback delivered over the simulator audio channel.
We pre-generated a library of short audio clips using a text-to-speech model, organized into two categories: (i) \emph{control suggestions} indicating the corrective direction along each control axis (e.g., \emph{``yaw slightly left''}, \emph{``maintain current roll''}), and (ii) \emph{encouragement} reinforcing successful maneuvers (e.g., \emph{``you are doing well''}).
At runtime, clip selection is driven by the per-axis discrepancy $|a^\human - \pi^*(o^{\ai})|$: when an axis exceeds a directional-error threshold, the corresponding control suggestion is queued; otherwise an encouragement clip is selected.
A 5-second cooldown is enforced between successive clips to prevent auditory overload and to give the learner time to act on each cue.

\section{User Study Details}
\label{app:human_study}

\subsection{Experimental Phases}
We structure each trial into three phases: a pre- and post-coaching test (2 laps each), and a main coached training session.
Since the majority of our participants are novices (have never operated a drone), learning to simultaneously control all four axes (thrust and roll, pitch, and yaw rates) of a racing drone in 40 minutes would be prohibitively difficult.
We therefore restrict the human's control authority to yaw and roll rate while automating the pitch rate and thrust; in effect, the human is only in charge of steering.

In Phase 1 (pre-coaching test), participants will complete an unassisted flight to assess their initial drone-racing skill. The phase ends after two completed laps.

In Phase 2 (coaching), participants will learn to fly the racing drone with assistance from their randomly assigned AI coach.
To prevent participant fatigue, we limit the coaching session to 15 laps (around 40 minutes, depending on learner capability).
Every three coached laps, the coach performs an unassisted \emph{evaluation lap} in which the learner flies alone, providing a clean signal for updating the coach's belief over learner skill level $\bel_t^\ell$ at gate $\ell$.

In Phase 3 (post-coaching test), participants complete two unassisted evaluation laps, mirroring the Phase 1 pre-coaching test. The difference between pre- and post-test performance is our primary measure of coaching-induced skill change.

\subsection{Initial Group Balance}
\label{app:initial-group-balance}
Based on self-reported drone-operating experience measured on a five-point scale (Appendix~\ref{app:questionnaire}), we found no statistically significant differences in initial skill levels across the three groups.
In addition, from Phase 1 data, we verified that the three randomly assigned groups achieved comparable initial objective performance. \autoref{fig:initial-group-balance} reports pre-coaching lap time and pre-coaching total failure count for all participants ($N=33$).
Through a one-way permutation ANOVA test, we found no reliable group difference for either lap time ($p=0.192$) or failure count ($p=0.663$).

\begin{figure}[H]
    \centering
    \includegraphics[width=0.80\linewidth]{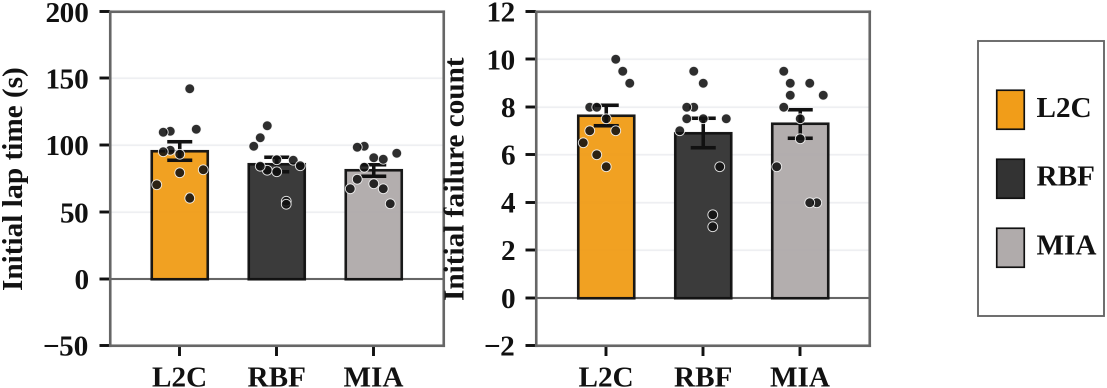}
    \caption{Initial group balance before coaching. \textit{Left:} pre-coaching lap time. \textit{Right:} pre-coaching total failure count. Bars show mean $\pm$ SEM, and dots show individual participants.}
    \label{fig:initial-group-balance}
\end{figure}

\subsection{Questionnaire}
\label{app:questionnaire}
Prior to the trial, participants self-reported their drone-piloting experience on a five-point scale:
\begin{enumerate}[leftmargin=2em, topsep=2pt, itemsep=0pt]
    \item No Experience: Have not operated a drone; have not played drone racing games, flight simulation games, or similar games using a controller.
    \item Casual Experience: Have occasionally operated consumer drones in low-speed scenarios (e.g., photography), or have played flight simulation games with a controller
    \item Regular Experience: Have regularly operated consumer drones, or have regularly played flight simulator games using a controller.
    \item Extensive Experience: Have regularly flown FPV drones, or have regularly practiced drone racing simulators to a proficient level (e.g., completing technical tracks cleanly at pace). Has not competed in organized races.
    \item Competitive Experience: Have competed in organized drone racing events, or regularly trains on drone racing simulators.
\end{enumerate}
Based on the participant feedback,
our pool consisted primarily of novices: 61.1\% reported No Experience, 36.1\% reported Casual Experience, and 2.8\% reported Regular Experience, with no participant reporting Extensive or Competitive Experience.
A one-way ANOVA confirmed no statistically significant differences in initial self-reported skill across the three groups.

\subsubsection{Questionnaire Metrics}
\label{app:questionnaire-metrics}
After Phase 3 concludes,
we collect a post-study questionnaire (5-point Likert scale) assessing subjective learning experience on perceived improvement, agency, and satisfaction (see \cref{tab:questionnaire-metrics}).
The results are shown in \cref{fig:experimental-results}.

\begin{table}[H]
    \centering
    \caption{Post-study survey items.}
    \label{tab:questionnaire-metrics}
    \small
    \begin{tabular}{@{}p{2.3cm}p{7.7cm}p{2.9cm}@{}}
        \toprule
        \textbf{Metric} & \textbf{Questionnaire item} & \textbf{Scoring} \\
        \midrule
        Safety & Compared to my first race, I flew more safely in the final race. & Higher is better \\
        \addlinespace
        Performance & I believe my overall drone racing performance
        has improved since my first race. & Higher is better \\
        \addlinespace
        Agency & The AI coach intervened too much throughout the race. & Reverse-coded; higher plotted value is better \\
        \addlinespace
        Satisfaction & Overall, I am happy with how the training went. & Higher is better \\
        \bottomrule
    \end{tabular}
\end{table}

\normalsize

\section{Proof of \cref{prop:skill_to_voi_equivalence_main}}
\label{app:proof_main}

\p{Setup}
Let $\hstate_t \in \hsset = \{\hstate_1,\ldots,\hstate_N\}$ denote the learner's discrete latent skill level at time $t$, ordered so that
$
\hstate_1 < \hstate_2 < \cdots < \hstate_N
$.
For any physical state $\state \in \sset$ and skill level $\hstate \in \hsset$, define the learner's \emph{independent continuation value}
\begin{equation}
\label{eq:ind_value_equivalence}
\valfunc^{\ind}(\state,\hstate)
\,\triangleq\,
\expectation\!\left[
    \sum_{k=0}^{\infty}\discount^k\,\reward^\task(\state_{t+k},\action^\human_{t+k})
    \;\middle|\;
    \begin{array}{l}
        \state_t=\state,\ \hstate_t=\hstate,\\[2pt]
        \action^\ai_{t+k}\equiv 0,\ \forall k\ge 0
    \end{array}
\right],
\end{equation}
namely, the discounted return the learner would obtain if the coach withdrew assistance from $(\state,\hstate)$ onward.
Fix a policy-independent evaluation distribution $\evaldist \in \Delta(\sset)$ and define the \emph{evaluation-averaged \gls{VoI}}
\begin{equation}
\label{eq:eval_averaged_voi_equivalence}
\bar{\reward}^\ai(\hstate)
\,\triangleq\,
\expectation_{\state\sim\evaldist}\!\left[\valfunc^{\ind}(\state,\hstate)\right].
\end{equation}
We assume $\reward^\task$ is uniformly bounded, so $\bar{\reward}^\ai$ is bounded as well.
For the analysis below, we focus on a simplified, yet pedagogically relevant regime of the hybrid \gls{PFA} in which skill does not decrease within an episode ($\betaS = 0$).
Concretely, from skill level $\hstate_i$ with $i<N$, the learner either remains at $\hstate_i$ or advances to $\hstate_{i+1}$ in one step; the top skill level $\hstate_N$ is absorbing.
If we also consider downskilling, the proof becomes more delicate: the key monotonicity step in \cref{lem:kernel_to_marginal_equivalence} is no longer automatic.
However, one may impose additional conditions, such as stochastic monotonicity of the full skill-transition kernel, to recover the same dominance conclusion.
Let $p_{\mathrm{succ}}^{\policy^\ai}(\hstate_i)$ denote the success probability under coach $\policy^\ai$ at skill $\hstate_i$.
Conditional on success, the learner upskills with probability $\alphaS(\hstate_i)$; conditional on failure, the learner upskills with probability $\alphaF(\hstate_i)$.
The resulting one-step skill-transition kernel under coach $\policy^\ai$ is therefore
\begin{equation}
\label{eq:pfa_kernel_pi}
\pfakernel^{\policy^\ai}(\hstate_{i+1}\mid \hstate_i)
=
p_{\mathrm{succ}}^{\policy^\ai}(\hstate_i)\,\alphaS(\hstate_i)
+
\bigl(1-p_{\mathrm{succ}}^{\policy^\ai}(\hstate_i)\bigr)\,\alphaF(\hstate_i),
\qquad i=1,\ldots,N-1,
\end{equation}
with
\[
\pfakernel^{\policy^\ai}(\hstate_i\mid \hstate_i)
=
1-\pfakernel^{\policy^\ai}(\hstate_{i+1}\mid \hstate_i),
\qquad
\pfakernel^{\policy^\ai}(\hstate_N\mid \hstate_N)=1.
\]
We define the \textit{effective upskill probability} as
$p_{\uparrow}^{\policy^\ai}(\hstate_i)
\,\triangleq\,
\pfakernel^{\policy^\ai}(\hstate_{i+1}\mid \hstate_i)$
for $i=1,\ldots,N-1$.

\begin{assumption}
\label{asm:V_monotone_equivalence}
$\bar{\reward}^\ai(\hstate)$ is nondecreasing in $\hstate\in\hsset$.
\end{assumption}

\Cref{asm:V_monotone_equivalence} states that more skilled learners are no worse off on average when left to act alone, reflecting the basic coaching premise that increased skill corresponds to improved independent competence. 
Note that a more interpretable sufficient condition is the pointwise inequality $\valfunc^{\ind}(\state,\hstate)\le \valfunc^{\ind}(\state,\hstate')$ for every $\state\in\sset$ and every pair $\hstate\le\hstate'$ in $\hsset$, which implies \cref{asm:V_monotone_equivalence} by taking expectation over $\state\sim\evaldist$.

We first show that local one-step kernel dominance propagates to all future skill marginals.

\begin{lemma}
\label{lem:kernel_to_marginal_equivalence}
If two coach policies $\policy^\ai$ and $\tilde{\policy}^\ai$ have the same initial distribution, and $\policy^\ai$ dominates $\tilde{\policy}^\ai$, \ie, $p_{\uparrow}^{\policy^\ai}(\hstate_i) \ge p_{\uparrow}^{\tilde{\policy}^\ai}(\hstate_i)$ for $i\in\{1,\ldots,N-1\}$,
then the induced skill marginals satisfy first-order stochastic dominance (FOSD) at every time:
\begin{equation}
\label{eq:marginal_fosd_equivalence}
\prob^{\policy^\ai}(\hstate_t \ge \hstate_j)
\ge
\prob^{\tilde{\policy}^\ai}(\hstate_t \ge \hstate_j),
\qquad
\forall j\in\{1,\ldots,N\},\ \forall t\ge 0.
\end{equation}
Equivalently, for every nondecreasing function $\phi:\hsset\to\reals$,
\[
\expectation_{\policy^\ai}[\phi(\hstate_t)]
\ge
\expectation_{\tilde{\policy}^\ai}[\phi(\hstate_t)],
\qquad \forall t\ge 0.
\]
\end{lemma}

\begin{proof}
We prove the equivalent test-function characterization by induction on $t$.

\emph{Base case.}
At $t=0$, the claim holds because the initial skill distributions are equal by assumption.

\emph{Inductive step.}
Assume that for some $t\ge 0$,
$
\expectation_{\policy^\ai}[\phi(\hstate_t)]
\ge
\expectation_{\tilde{\policy}^\ai}[\phi(\hstate_t)]$
for every nondecreasing $\phi:\hsset\to\reals$.
Let $\phi:\hsset\to\reals$ be any nondecreasing function.
Define
\[
g^{\policy^\ai}(\hstate)
\triangleq
\expectation_{\hstate'\sim \pfakernel^{\policy^\ai}(\cdot\mid \hstate)}[\phi(\hstate')],
\qquad
g^{\tilde{\policy}^\ai}(\hstate)
\triangleq
\expectation_{\hstate'\sim \pfakernel^{\tilde{\policy}^\ai}(\cdot\mid \hstate)}[\phi(\hstate')].
\]
By the law of total expectation, the expected value of $\phi(\hstate_{t+1})$ can be computed by first conditioning on the current skill $\hstate_t$, and then averaging over $\hstate_t$.
Therefore, by definition of $g^{\policy^\ai}$, we have
$\mathbb{E}[\phi(\hstate_{t+1}) \mid \hstate_t] = g^{\policy^\ai}(\hstate_t)$, which yields:
\begin{equation}
\label{eq:tower_rule_equivalence}
\expectation_{\policy^\ai}[\phi(\hstate_{t+1})]
=
\expectation_{\policy^\ai}[g^{\policy^\ai}(\hstate_t)],
\qquad
\expectation_{\tilde{\policy}^\ai}[\phi(\hstate_{t+1})]
=
\expectation_{\tilde{\policy}^\ai}[g^{\tilde{\policy}^\ai}(\hstate_t)].
\end{equation}

We now establish two facts.

\emph{Fact 1:} \(g^{\policy^\ai}(\hstate)\ge g^{\tilde{\policy}^\ai}(\hstate)\) for every \(\hstate\in\hsset\).

For $i<N$, both kernels are supported on $\{\hstate_i,\hstate_{i+1}\}$, and $\policy^\ai$ dominates $\tilde{\policy}^\ai$ implies that
$
\pfakernel^{\policy^\ai}(\cdot\mid \hstate_i)
\succeq_{\mathrm{FOSD}}
\pfakernel^{\tilde{\policy}^\ai}(\cdot\mid \hstate_i)
$.
Since $\phi$ is nondecreasing, taking expectations preserves the order:
$
g^{\policy^\ai}(\hstate_i)\ge g^{\tilde{\policy}^\ai}(\hstate_i)
$.
At the absorbing top state $\hstate_N$, both kernels are point masses at $\hstate_N$, so equality holds.

\emph{Fact 2:} \(g^{\tilde{\policy}^\ai}\) is nondecreasing on \(\hsset\).

For $i<N$,
$
g^{\tilde{\policy}^\ai}(\hstate_i)
=
\bigl(1-p_{\uparrow}^{\tilde{\policy}^\ai}(\hstate_i)\bigr)\phi(\hstate_i)
+
p_{\uparrow}^{\tilde{\policy}^\ai}(\hstate_i)\phi(\hstate_{i+1})
$,
so \(g^{\tilde{\policy}^\ai}(\hstate_i)\) is a convex combination of \(\phi(\hstate_i)\) and \(\phi(\hstate_{i+1})\).
Therefore,
$
\phi(\hstate_i)
\le
g^{\tilde{\policy}^\ai}(\hstate_i)
\le
\phi(\hstate_{i+1})
$.
Likewise,
$
\phi(\hstate_{i+1})
\le
g^{\tilde{\policy}^\ai}(\hstate_{i+1})
\le
\phi(\hstate_{i+2})$ for $i\le N-2$,
and at the top state,
$
g^{\tilde{\policy}^\ai}(\hstate_N)=\phi(\hstate_N)
$.
Hence
$
g^{\tilde{\policy}^\ai}(\hstate_i)
\le
\phi(\hstate_{i+1})
\le
g^{\tilde{\policy}^\ai}(\hstate_{i+1})$ for $i=1,\ldots,N-1$,
so \(g^{\tilde{\policy}^\ai}\) is nondecreasing.

Using \eqref{eq:tower_rule_equivalence}, Facts 1--2, and the inductive hypothesis with the nondecreasing test function \(g^{\tilde{\policy}^\ai}\), we obtain
\begin{align*}
\expectation_{\policy^\ai}[\phi(\hstate_{t+1})]
&=
\expectation_{\policy^\ai}[g^{\policy^\ai}(\hstate_t)]
&& \text{(Law of total expectation)} \\
&\ge
\expectation_{\policy^\ai}[g^{\tilde{\policy}^\ai}(\hstate_t)]
&& \text{(Fact 1)} \\
&\ge
\expectation_{\tilde{\policy}^\ai}[g^{\tilde{\policy}^\ai}(\hstate_t)]
&& \text{(Induction + Fact 2: \(g^{\tilde{\policy}^\ai}\) is nondecreasing)} \\
&=
\expectation_{\tilde{\policy}^\ai}[\phi(\hstate_{t+1})]
&& \text{(Law of total expectation)}
\end{align*}
The claim follows for \(t+1\).
\end{proof}

\begin{remark}
Inequality
$p_{\uparrow}^{\policy^\ai}(\hstate_i) \ge p_{\uparrow}^{\tilde{\policy}^\ai}(\hstate_i)$ essentially says that from any current skill level, coach $\policy^\ai$ induces a weakly more favorable next-step skill distribution than coach $\tilde{\policy}^\ai$.
Importantly, this condition does \emph{not} force the better coach to maximize immediate success, because both success and failure may produce upskilling through $\alphaS$ and $\alphaF$, so a coach may satisfy $p_{\uparrow}^{\policy^\ai}(\hstate_i) \ge p_{\uparrow}^{\tilde{\policy}^\ai}(\hstate_i)$ either by increasing success when success is more pedagogically productive, or by strategically allowing informative failures when failure is more productive.

A sufficient condition for $p_{\uparrow}^{\policy^\ai}(\hstate_i) \ge p_{\uparrow}^{\tilde{\policy}^\ai}(\hstate_i)$ is the pointwise inequality
\begin{equation*}
\Bigl(
p_{\mathrm{succ}}^{\policy^\ai}(\hstate_i)
-
p_{\mathrm{succ}}^{\tilde{\policy}^\ai}(\hstate_i)
\Bigr)
\Bigl(
\alphaS(\hstate_i)-\alphaF(\hstate_i)
\Bigr)
\ge 0,
\qquad i=1,\ldots,N-1,
\end{equation*}
obtained by subtracting \eqref{eq:pfa_kernel_pi} for the two policies.
The condition states that the change in success probability induced by the coach aligns with the relative pedagogical value of success versus failure events, thereby increasing the learner's overall upskilling probability.
\end{remark}

We can now state the desired connection between skill dominance and \gls{VoI} dominance.

\setcounter{proposition}{0}
\begin{proposition}[Skill dominance implies \gls{VoI} dominance]
Suppose \cref{asm:V_monotone_equivalence} holds. If a coaching policy $\policy^\ai$ induces a weakly higher one-step upskill probability than $\tilde{\policy}^\ai$ at every skill level, \ie, $p_{\uparrow}^{\policy^\ai}(\hstate_i)\ge p_{\uparrow}^{\tilde{\policy}^\ai}(\hstate_i)$ for $i\in\{1,\ldots,N-1\}$, then
\begin{equation}
\expectation_{\policy^\ai}\!\left[
    \sum_{t=0}^{\infty}\discount^t\,\bar{\reward}^\ai(\hstate_t)
\right]
\ge
\expectation_{\tilde{\policy}^\ai}\!\left[
    \sum_{t=0}^{\infty}\discount^t\,\bar{\reward}^\ai(\hstate_t)
\right].
\end{equation}
\end{proposition}

\begin{proof}
By \cref{asm:V_monotone_equivalence}, $\bar{\reward}^\ai$ is nondecreasing on $\hsset$.
Next, by \cref{lem:kernel_to_marginal_equivalence}, for every \(t\ge 0\),
\[
\hstate_t^{\policy^\ai}
\succeq_{\mathrm{FOSD}}
\hstate_t^{\tilde{\policy}^\ai}.
\]
Applying the test-function characterization of FOSD with the nondecreasing function \(\bar{\reward}^\ai\) yields
\[
\expectation_{\policy^\ai}[\bar{\reward}^\ai(\hstate_t)]
\ge
\expectation_{\tilde{\policy}^\ai}[\bar{\reward}^\ai(\hstate_t)],
\qquad \forall t\ge 0.
\]
Multiplying by \(\discount^t\ge 0\) and summing over \(t\) preserves the inequality.
Because $\bar{\reward}^\ai$ is bounded and $\discount \in [0,1)$, the discounted series is absolutely convergent, so we can exchange the sum and expectation using Fubini's theorem, which gives
$
\expectation_{\policy^\ai}\!\left[
    \sum_{t=0}^{\infty}\discount^t\,\bar{\reward}^\ai(\hstate_t)
\right]
\ge
\expectation_{\tilde{\policy}^\ai}\!\left[
    \sum_{t=0}^{\infty}\discount^t\,\bar{\reward}^\ai(\hstate_t)
\right]
$.
\end{proof}

\end{document}

%% file: headers/preamble.tex
\usepackage[shortlabels]{enumitem}
\usepackage{needspace}
\usepackage{subcaption}
\usepackage{colortbl}
\usepackage[xindy, acronyms, nowarn]{glossaries}
\usepackage{diagbox}

\usepackage{algorithm}
\usepackage{algorithmic}
\usepackage{wrapfig}

\usepackage{booktabs,tabularx,array,float} %

\usepackage{bbm, amsmath, amssymb, amsthm}  %
\usepackage{xcolor}
\usepackage{lipsum}
\usepackage{graphicx}
\usepackage[font=footnotesize]{caption}
\usepackage{layout}
\usepackage{verbatim}  %
\usepackage{indentfirst}
\usepackage{thmtools}
\usepackage{enumitem}
\usepackage{pifont} %
\usepackage{float}
\usepackage{wasysym}

\usepackage{array}
\usepackage{multirow}
\usepackage{longtable}  %
\usepackage{booktabs}  %
\usepackage{threeparttable}

\usepackage{hyperref}
\hypersetup{
 colorlinks=true,
 linkcolor=darkblue,
 filecolor=darkblue,
 citecolor=darkblue,      
 urlcolor=black,
 }
\makeatletter
\hypersetup{pdftitle=\@title,pdfauthor=\@author}
\makeatother

\usepackage[
    nameinlink,
    capitalize,  %
    noabbrev
]{cleveref}
\crefformat{equation}{#2(#1)#3}
\Crefformat{equation}{#2Equation~(#1)#3}

\usepackage[breakable]{tcolorbox}

\definecolor{porange}{RGB}{231, 117, 0}  %
\definecolor{nvgreen}{RGB}{118, 185, 0}  %
\definecolor{qualblue}{RGB}{50, 83, 220}  %
\definecolor{darkgreen}{RGB}{29, 177, 2}

\usepackage{amsthm}

\newtheorem{lemma}{Lemma}    

\newtheorem{remark}{Remark}

\newtheorem{proposition}{Proposition}    
\newtheorem{assumption}{Assumption}    
\newtheorem{definition}{Definition}

\DeclareDocumentEnvironment{example}{}{\textit{Running example:}}{}

\AtBeginEnvironment{proof}{\small}
\AtBeginEnvironment{example}{\small}
\AtBeginEnvironment{table}{\small}

\usepackage{ifthen}
\newboolean{include-notes}
\newboolean{mark-new}
\newboolean{include-remove}
\newboolean{include-frontmatter} %
\newboolean{hide-chapters}

\setboolean{include-frontmatter}{true}
\setboolean{include-notes}{true}
\setboolean{mark-new}{false}
\setboolean{include-remove}{false}

\usepackage[normalem]{ulem}
\definecolor{exampleback}{rgb}{0.6, 0.6, 0.6}
\definecolor{teal}{rgb}{0, 0.5, 0.5}
\definecolor{darkgreen}{RGB}{0,100,0}
\definecolor{darkblue}{RGB}{0,45,114}
\definecolor{porange}{RGB}{231, 117, 0}  %
\definecolor{turquoise}{RGB}{78, 190, 186}  %
\definecolor{lightpink}{HTML}{F1AFF7}  %
\newcommand{\haimin}[1]{\ifthenelse{\boolean{include-notes}}{\textcolor{magenta}{\textbf{HH: #1}}}{}}
\newcommand{\wei}[1]{\ifthenelse{\boolean{include-notes}}{\textcolor{porange}{\textbf{Wei: #1}}}{}}

\newcommand{\remove}[1]{\ifthenelse{\boolean{include-remove}}{\textcolor{red}{\sout{#1}}}{}}
\newcommand{\new}[1]{\ifthenelse{\boolean{mark-new}}{\textcolor{teal}{#1}}{#1}}

\newcommand\blfootnote[1]{%
  \begingroup
  \renewcommand\thefootnote{}\footnote{#1}%
  \addtocounter{footnote}{-1}%
  \endgroup
}

\newcolumntype{C}[1]{>{\centering\arraybackslash}m{#1} } %
\renewcommand{\arraystretch}{1.3}

\usepackage{tcolorbox}
\tcbuselibrary{skins}

\newtcolorbox{l2cquote}{
  enhanced,
  colback=orange!15,        %
  colframe=orange!70!black,
  arc=3mm,                  %
  boxrule=0.5pt,
  left=2mm, right=2mm, top=1.5mm, bottom=1.5mm,
  fontupper=\small,
  before skip=4pt, after skip=4pt,
}

\newtcolorbox{rbfquote}{
  enhanced,
  colback=black!10,
  colframe=black!70,
  arc=3mm,
  boxrule=0.5pt,
  left=2mm, right=2mm, top=1.5mm, bottom=1.5mm,
  fontupper=\small,
  before skip=4pt, after skip=4pt,
}

\newtcolorbox{miaquote}{
  enhanced,
  colback=gray!15,
  colframe=gray!60,
  arc=3mm,
  boxrule=0.5pt,
  left=2mm, right=2mm, top=1.5mm, bottom=1.5mm,
  fontupper=\small,
  before skip=4pt, after skip=4pt,
}

%% file: headers/gls.tex
\glsdisablehyper
\makeglossaries

\newglossaryentry{FPV}
{
  name={FPV},
  description={first-person view},
  first={first-person view (\glsentrytext{FPV})},
}

\newglossaryentry{POSG}
{
  name={POSG},
  plural={POSGs},
  description={partially observable stochastic game},
  first={partially observable stochastic game (\glsentrytext{POSG})},
  descriptionplural={partially observable stochastic games},
  firstplural={partially observable stochastic games (\glsentryplural{POSGs})}
}

\newglossaryentry{POMDP}
{
    name={POMDP},
    plural={POMDPs},
    description={partially observable Markov decision process},
    first={partially observable Markov decision process (\glsentrytext{POMDP})},
    descriptionplural={partially observable Markov decision processes},
    firstplural={partially observable Markov decision processes (\glsentryplural{POMDPs})}
}

\newglossaryentry{VoI}
{
  name={VoI},
  description={Value of Independence},
  first={Value of Independence (\glsentrytext{VoI})},
}

\newglossaryentry{L2C}
{
  name={L2C},
  description={Learning to Coach},
  first={Learning to Coach (\glsentrytext{L2C})},
}

\newglossaryentry{RBF}
{
  name={RBF},
  description={Rule-based Fading},
  first={Rule-based Fading (\glsentrytext{RBF})},
}

\newglossaryentry{MIA}
{
  name={MIA},
  description={Minimally-invasive Assistance},
  first={Minimally-invasive Assistance (\glsentrytext{MIA})},
}

\newglossaryentry{FCP}
{
  name={FCP},
  description={Fictitious Co-Play},
  first={Fictitious Co-Play (\glsentrytext{FCP})},
}

\newglossaryentry{PFA}
{
  name={PFA},
  description={probabilistic finite-state automaton},
  first={probabilistic finite-state automaton (\glsentrytext{PFA})},
}

\newglossaryentry{LSE}
{
  name={LSE},
  plural={LSEs},
  description={local Stackelberg equilibrium},
  first={local Stackelberg equilibrium (\glsentrytext{LSE})},
  descriptionplural={local Stackelberg equilibria},
  firstplural={local Stackelberg equilibria (\glsentryplural{LSE})}
}

\newglossaryentry{FSE}
{
  name={FSE},
  plural={FSEs},
  description={feedback Stackelberg equilibrium},
  first={feedback Stackelberg equilibrium (\glsentrytext{FSE})},
  descriptionplural={feedback Stackelberg equilibria},
  firstplural={feedback Stackelberg equilibria (\glsentryplural{FSE})}
}

\newglossaryentry{DSE}
{
  name={DSE},
  plural={DSEs},
  description={differential Stackelberg equilibrium},
  first={differential Stackelberg equilibrium (\glsentrytext{DSE})},
  descriptionplural={differential Stackelberg equilibria},
  firstplural={differential Stackelberg equilibria (\glsentryplural{DSE})}
}

\newglossaryentry{RL}
{
  name={RL},
  description={reinforcement learning},
  first={reinforcement learning (\glsentrytext{RL})},
}

\newglossaryentry{HJI}
{
  name={HJI},
  description={Hamilton--Jacobi--Isaacs},
  first={Hamilton--Jacobi--Isaacs (\glsentrytext{HJI})},
}

\newglossaryentry{GAS}
{
  name={GAS},
  description={globally asymptotically stable},
  first={globally asymptotically stable (\glsentrytext{GAS})},
}

\newglossaryentry{a.s.}
{
  name={a.s.},
  description={almost surely},
  first={almost surely (\glsentrytext{a.s.})},
}

\newglossaryentry{MDP}
{
  name={MDP},
  description={Markov decision process},
  first={Markov decision process (\glsentrytext{MDP})},
}

\newglossaryentry{SAC}
{
  name={SAC},
  description={Soft Actor--Critic},
  first={Soft Actor--Critic (\glsentrytext{SAC})},
}

\newglossaryentry{tGDA}
{
  name={$\tau$-GDA},
  description={gradient descent-ascent with finite timescale separation},
  first={gradient descent-ascent with finite timescale separation (\glsentrytext{tGDA})},
}

\newglossaryentry{MAGICS}
{
  name={MAGICS},
  description={Minimax Actors Guided by Implicit Critic Stackelberg},
  first={Minimax Actors Guided by Implicit Critic Stackelberg (\glsentrytext{MAGICS})},
}

\newglossaryentry{ISAACS}
{
  name={ISAACS},
  description={Iterative Soft Adversarial Actor--Critic for Safety},
  first={Iterative Soft Adversarial Actor--Critic for Safety (\glsentrytext{ISAACS})},
}

\newglossaryentry{ISAACStbg}
{
  name={ISAAC$\mathcal{S}$},
  description={Iterative Soft Adversarial Actor--Critic with Stackelberg learning dynamics for Safety},
  first={Iterative Soft Adversarial Actor--Critic with Stackelberg learning dynamics for Safety (\glsentrytext{ISAACStbg})},
}

\newglossaryentry{A2C}
{
  name={A2C},
  description={Advantage Actor--Critic},
  first={Advantage Actor-Critic (\glsentrytext{A2C})},
}

\newglossaryentry{iHvp}
{
  name={iHvp},
  description={inverse hessian vector product},
  first={inverse-Hessian-vector product
  (\glsentrytext{iHvp})},
}

\newglossaryentry{jvp}
{
  name={jvp},
  description={jacobian vector product},
  first={Jacobian-vector product
  (\glsentrytext{jvp})},
}

\newglossaryentry{GSE}
{
  name={GSE},
  description={global Stackelberg equilibrium},
  first={global Stackelberg equilibrium (\glsentrytext{GSE})}
}

\newglossaryentry{BNP}
{
  name={B\&P},
  description={Branch-and-Play},
  first={Branch-and-Play (\glsentrytext{BNP})},
}

\newglossaryentry{LQ}
{
  name={LQ},
  description={linear quadratic},
  first={linear quadratic (\glsentrytext{LQ})},
}

\newglossaryentry{ILQR}
{
  name={ILQR},
  description={iterative linear quadratic regulator},
  first={iterative linear quadratic regulator (\glsentrytext{ILQR})},
}

\newglossaryentry{ATC}
{
  name={ATC},
  description={air traffic control},
  first={air traffic control (\glsentrytext{ATC})},
}

\newglossaryentry{STP}
{
  name={STP},
  description={sequential trajectory planning},
  first={sequential trajectory planning (\glsentrytext{STP})},
}

\newglossaryentry{FCFS}
{
  name={FCFS},
  description={first-come-first-served},
  first={first-come-first-served (\glsentrytext{FCFS})},
}

\newglossaryentry{MAPF}
{
  name={MAPF},
  description={multi-agent pathfinding},
  first={Multi-agent pathfinding (\glsentrytext{MAPF})},
}

\newglossaryentry{MPC}
{
  name={MPC},
  description={model predictive control},
  first={model predictive control (\glsentrytext{MPC})},
}

\newglossaryentry{NOD}
{
  name={NOD},
  plural={NOD},
  description={nonlinear opinion dynamics},
  first={nonlinear opinion dynamics (\glsentrytext{NOD})},
  descriptionplural={nonlinear opinion dynamics},
  firstplural={nonlinear opinion dynamics (\glsentryplural{NOD})}
}

\newglossaryentry{DNN}
{
  name={DNN},
  description={deep neural network},
  first={deep neural network (\glsentrytext{DNN})},
}

\newglossaryentry{MLE}
{
  name={MLE},
  description={maximum likelihood estimation},
  first={maximum likelihood estimation (\glsentrytext{MLE})},
}

\newglossaryentry{ODG}
{
  name={ODG},
  description={opinion-guided dynamic game},
  first={opinion-guided dynamic game (\glsentrytext{ODG})},
}

\newglossaryentry{MLP}
{
  name={MLP},
  description={multilayer perceptron},
  first={multilayer perceptron (\glsentrytext{MLP})},
}

\newglossaryentry{E2E-BC}
{
  name={E2E-BC},
  description={end-to-end behavior cloning},
  first={End-to-end behavior cloning (\glsentrytext{E2E-BC})},
}

\newglossaryentry{HJ}
{
  name={HJ},
  description={Hamilton--Jacobi},
  first={Hamilton--Jacobi (\glsentrytext{HJ})}
}

\newglossaryentry{DCBF}
{
  name={DCBF},
  description={discrete-time control barrier function},
  first={discrete-time control barrier function (\glsentrytext{DCBF})}
}

\newglossaryentry{CBF}
{
  name={CBF},
  description={control barrier function},
  first={control barrier function (\glsentrytext{CBF})}
}

\newglossaryentry{Q-CBF}
{
  name={Q-CBF},
  description={state--action control barrier function},
  first={state--action control barrier function (\glsentrytext{Q-CBF})}
}

\newglossaryentry{ODD}
{
  name={ODD},
  description={Operational Design Domain},
  first={operational design domain (\glsentrytext{ODD})}
}

\newglossaryentry{LRSF}
{
  name={LRSF},
  description={Last-Resort Safety Filter},
  first={last-resort safety filter (\glsentrytext{LRSF})}
}

\newglossaryentry{HCSF}
{
  name={HCSF},
  description={Human-Centered Safety Filter},
  first={human-centered safety filter (\glsentrytext{HCSF})}
}

\newglossaryentry{NLP}
{
  name={NLP},
  description={nonlinear programming problem},
  first={nonlinear programming problem (\glsentrytext{NLP})},
}

\newglossaryentry{AC}
{
  name={AC},
  description={Assetto Corsa},
  first={Assetto Corsa (\glsentrytext{AC})},
}

\newglossaryentry{ANOVA}
{
    name={ANOVA},
    description={Analysis of Variance},
    first={analysis of variance (\glsentrytext{ANOVA})}
}

\newglossaryentry{SME}
{
    name={SME},
    description={Simple Main Effects},
    first={simple main effects (\glsentrytext{SME})},
}

\newglossaryentry{HSD}
{
    name={HSD},
    description={Honestly Significant Difference},
    first={honestly significant difference (\glsentrytext{HSD})},
}

\newglossaryentry{ECDF}
{
    name={ECDF},
    description={Empirical Cumulative Distribution Function},
    first={empirical cumulative distribution function (\glsentrytext{ECDF})},
}

\newglossaryentry{OCP}
{
    name={OCP},
    description={Optimal Control Problem},
    first={optimal control problem (\glsentrytext{OCP})}
}

\newglossaryentry{AI}
{
    name={AI},
    description={Artificial Intelligence},
    first={artificial intelligence (\glsentrytext{AI})}
}

\newglossaryentry{HRI}
{
    name={HRI},
    description={Human--Robot Interaction},
    first={human--robot interaction (\glsentrytext{HRI})}
}

%% file: headers/notation.tex
\usepackage{fontawesome5}  %

\newcommand{\ourgame}{{\textit{Coaching Game}}}
\newcommand{\ai}{{C}}
\newcommand{\transprob}{{T}}
\newcommand{\hstate}{{\theta}} %
\newcommand{\hsset}{{\Theta}}
\newcommand{\blend}{{\lambda}}

\newcommand{\pindex}{\mathcal{I}} %
\newcommand{\history}{{h}}
\newcommand{\ind}{{\text{ind}}}
\newcommand{\evaldist}{\distr_{\mathrm{eval}}}

\newcommand{\pfakernel}{P}

\newcommand{\alphaS}{\alpha_S}
\newcommand{\alphaF}{\alpha_F}
\newcommand{\betaS}{\beta_S}

\newcommand{\p}[1]{\smallskip \noindent \textbf{{#1}.}}

\DeclareMathOperator*{\argmax}{{\mathop{\mathrm{argmax}}}}

\newcommand{\reals}{\mathbb{R}}

\newcommand{\distr}{p}
\newcommand{\prob}{{P}}

\newcommand{\bel}{b}
\newcommand{\belspace}{\Delta}
\DeclareMathOperator*{\expectation}{\mathbb{E}}

\newcommand{\state}{{s}}
\newcommand{\action}{{a}}
\newcommand{\jstate}{\bar{s}}

\newcommand{\obsset}{{\mathcal{O}}}
\newcommand{\obsfunc}{{g}}

\newcommand{\sset}{{\mathcal{S}}}
\newcommand{\aset}{{\mathcal{A}}}

\newcommand{\game}{{\mathcal{G}}}

\newcommand{\valfunc}{{V}}

\newcommand{\qfunc}{{Q}}
\newcommand{\qfuncopt}{\qfunc^*}

\newcommand{\policy}{{\pi}}
\newcommand{\policyopt}{\policy^*}

\newcommand{\human}{{L}}

\newcommand{\reward}{{r}}

\newcommand{\discount}{{\gamma}}

\newcommand{\task}{{\text{task}}}

\newcommand{\obs}{{o}}
\newcommand{\obsh}{\obs^\human}
\newcommand{\obsr}{\obs^\ai}

\newcommand{\eg}{\text{e.g.}}
\newcommand{\ie}{\text{i.e.}}

\newcommand{\princeton}[1]{\ifthenelse{\boolean{include-notes}}{\textcolor{orange}{#1}}{}}